\documentclass[11pt,english]{article}
\usepackage[utf8]{inputenc}
\usepackage{amsmath}
\usepackage{amsthm}
\usepackage{amssymb}
\usepackage{amsfonts}
\usepackage{comment}
\usepackage{mathtools}
\usepackage{algorithm}
\usepackage{graphicx}
\usepackage[font=footnotesize]{caption}
\usepackage{xcolor}
\usepackage{enumitem}
\usepackage{nicefrac}
\DeclareMathOperator*{\argmin}{arg\,min}
\usepackage{bbding}

\usepackage{a4wide}
\usepackage{chngpage}
\usepackage{authblk}

\newtheorem{fact}{Fact}

\usepackage[T1]{fontenc}
\usepackage{babel}
\usepackage{threeparttable}
\usepackage{booktabs}
\usepackage{etoolbox}
\appto\TPTnoteSettings{\footnotesize}
\usepackage[font=footnotesize]{caption}
\usepackage{pifont}

\usepackage{mathtools, nccmath, textcomp}
\usepackage[round]{natbib}

\usepackage[backref=page]{hyperref}
\renewcommand*{\backref}[1]{}
\renewcommand*{\backrefalt}[4]{%
    \ifcase #1 %
    \or        (Cited on page~#2.)%
    \else      (Cited on pages~#2.)%
    \fi}

\usepackage{xcolor} 
\usepackage{amsfonts, amsthm, amssymb} 
\usepackage{mathtools}

\usepackage{subcaption}
\usepackage{makecell}
\usepackage{enumitem}
\usepackage{nicefrac}

\usepackage[normalem]{ulem}
\usepackage[normalem]{ulem}
\usepackage{algorithm}
\usepackage[noend]{algcompatible}
\newcommand{\AT}[1]{\State \textbf{at} #1 \textbf{do}}

\usepackage{hyperref, tabularx}
\usepackage[capitalise]{cleveref}
\makeatletter
\newcommand{\multiline}[1]{%
  \begin{tabularx}{\dimexpr\linewidth-\ALG@thistlm}[t]{@{}X@{}}
    #1
  \end{tabularx}
}
\makeatother
\algnewcommand{\Initialize}[1]{%
  \State \textbf{Initialize:} #1}
\theoremstyle{plain}
\newtheorem{theorem}{Theorem}
\newtheorem*{theorem*}{Theorem}

\newtheorem{lemma}{Lemma}[section]

\newtheorem*{cor*}{Corollary}
\theoremstyle{definition}
\newtheorem{definition}{Definition}
\newtheorem{assump}{Assumption}
\theoremstyle{remark}
\newtheorem{remark}{Remark}

\newcommand{\mbb}{\mathbb}

\newcommand{\mbe}{\mathbb E}

\newcommand{\lp}{\left(}
\newcommand{\rp}{\right)}

\newcommand{\lcb}{\left\{}
\newcommand{\rcb}{\right\}}
\newcommand{\lbr}{\left[}
\newcommand{\rbr}{\right]}
\newcommand{\lnr}{\left\|}
\newcommand{\rnr}{\right\|}

\newcommand\norm[1]{\lnr#1\rnr}

\newcommand{\bx}{{\mathbf x}}
\newcommand{\bxt}{{\mathbf x^{(t)}}}
\newcommand{\bxk}{{\mathbf x^{(k)}}}
\newcommand{\bxtp}{{\mathbf x^{(t+1)}}}

\newcommand{\be}{{\mathbf e}}
\newcommand{\bet}{{\mathbf e^{(t)}}}
\newcommand{\bek}{{\mathbf e^{(k)}}}

\newcommand{\bz}{{\mathbf z}}
\newcommand{\bzt}{{\mathbf z^{(t)}}}

\newcommand{\bPsi}{{\boldsymbol{\Psi}}}

\newcommand{\bPhi}{{\boldsymbol{\Phi}}}
\newcommand{\bPhit}{{\bPhi^{(t)}}}

\newcommand{\by}{{\mathbf y}}

\newcommand{\G}{\nabla}

\newcommand{\nn}{\nonumber}

\usepackage{xcolor}

\usepackage{tikz}
\usetikzlibrary{calc,trees,positioning,arrows,chains,shapes.geometric,%
	decorations.pathreplacing,decorations.pathmorphing,shapes,%
	matrix,shapes.symbols}

\tikzstyle{startstop} = [rectangle, draw, rounded corners, align=center, minimum width=3cm, minimum height=1cm,text centered]
\tikzstyle{decision} = [diamond, draw, fill=blue!20, 
text width=4.5em, text badly centered, node distance=3cm, inner sep=0pt]
\tikzstyle{block} = [rectangle, draw, fill=blue!10, align=center, rounded corners, minimum width=3cm, minimum height=1cm]
\tikzstyle{blockcast} = [rectangle, draw, fill=red!10, align=center, rounded corners, minimum width=3cm, minimum height=0.45cm]
\tikzstyle{line} = [draw, -latex']
\tikzstyle{cloud} = [draw, ellipse,fill=red!20, node distance=3cm,
minimum height=2em]

\newcommand{\R}{\mathbb{R}}
\newcommand{\E}{\mathbb{E}}
\newcommand{\N}{\mathbb{N}}

\newcommand{\J}{\mathcal{J}}
\newcommand{\calN}{\mathcal{N}}

\newtheorem{example}{Example}

\makeatletter
\renewcommand\@fnsymbol[1]{}
\makeatother

\usepackage{multirow}
\usepackage{tabularx}

\author[1]{Aleksandar Armacki
}
\author[1]{Pranay Sharma}
\author[1]{Gauri Joshi}
\author[2]{Dragana Bajovi\'{c}}
\author[3]{Du\v{s}an Jakoveti\'{c}}
\author[1]{Soummya Kar}
\affil[1]{Carnegie Mellon University, Pittsburgh, PA, USA\\ \texttt{\{aarmacki,pranaysh,gaurij,soummyak\}@andrew.cmu.edu }}
\affil[2]{Faculty of Technical Sciences, University of Novi Sad, Novi Sad, Serbia\\ \texttt{dbajovic@uns.ac.rs}}
\affil[3]{Faculty of Sciences, University of Novi Sad, Novi Sad, Serbia\\ \texttt{dusan.jakovetic@dmi.uns.ac.rs}}

\title{High-probability Convergence Bounds for Nonlinear Stochastic Gradient Descent Under Heavy-tailed Noise}
\date{}

\begin{document}

\maketitle

\begin{abstract}
    We study high-probability convergence guarantees of learning on streaming data in the presence of heavy-tailed noise. In the proposed scenario, the model is updated in an online fashion, as new information is observed, without storing any additional data. To combat the heavy-tailed noise, we consider a general framework of nonlinear stochastic gradient descent (SGD), providing several strong results. First, for non-convex costs and component-wise nonlinearities, we establish a convergence rate arbitrarily close to $\mathcal{O}\left(t^{-\nicefrac{1}{4}}\right)$, \emph{whose exponent is independent of noise and problem parameters}. Second, for strongly convex costs and component-wise nonlinearities, we establish a rate arbitrarily close to $\mathcal{O}\left(t^{-\nicefrac{1}{2}}\right)$ for the weighted average of iterates, with exponent again independent of noise and problem parameters. Finally, for strongly convex costs and a broader class of nonlinearities, we establish convergence of the last iterate, with a rate $\mathcal{O}\left(t^{-\zeta} \right)$, where $\zeta \in (0,1)$ depends on problem parameters, noise and nonlinearity. As we show analytically and numerically, $\zeta$ can be used to inform the preferred choice of nonlinearity for given problem settings. Compared to state-of-the-art, who only consider clipping, require bounded noise moments of order $\eta \in (1,2]$, and establish convergence rates whose exponents go to zero as $\eta \rightarrow 1$, we provide high-probability guarantees for a much broader class of nonlinearities and symmetric density noise, with convergence rates whose exponents are bounded away from zero, even when the noise has \emph{finite first moment only}. Moreover, in the case of strongly convex functions, we demonstrate analytically and numerically that clipping is not always the optimal nonlinearity, further underlining the value of our general framework.
\end{abstract}

\section{Introduction}

Learning on streaming data is a paradigm in which incoming samples are processed incrementally, while using limited memory and time, e.g.,~\cite{streams1,streams2}. Formally, it can be represented as an optimization problem, with the goal of solving
\begin{equation}\label{eq:problem}
    \argmin_{\bx \in \R^d} \lcb f(\bx) \triangleq \mbe_{\omega} [\ell(\bx; \omega)] \rcb,
\end{equation} where $\bx \in \R^d$ represents model parameters, $\ell: \R^d \times \mathcal{W} \mapsto \R$ is a loss function, while $\omega \in \mathcal{W}$ is a random sample. The function $f: \mbb R^d \mapsto \mbb R$ is commonly known as the \emph{population loss}. Many modern machine learning applications, such as classification and regression, can be modeled using \eqref{eq:problem}. Learning on streaming data can be seen as a special case of \emph{stochastic optimization} (SO), e.g.,~\cite{robbins1951stochastic,nemirovski2009robust,ghadimi2013stochastic}, with some important constraints. While the goal of both is to solve~\eqref{eq:problem}, SO approaches utilize both stochastic and mini-batch estimators, e.g.,~\cite{curtis-large-scale_ml,pmlr-v119-woodworth20a,woodworth_NEURIPS2020}, while streaming algorithms use only stochastic estimators and update the model parameters as each new sample arrives in the stream, without storing any additional information\footnote{Such as the samples themselves, or stochastic gradients.}, e.g.,~\cite{harvey2019tight,nguyen2023improved,jakovetic2023nonlinear}.

Perhaps the most popular method to solve~\eqref{eq:problem} is Stochastic Gradient Descent (SGD) \cite{robbins1951stochastic}, with its popularity stemming from low computation cost and incredible empirical success, e.g.,~\cite{bottou2010large, hardt2016train}. Theoretical convergence guarantees of SGD have been studied extensively, e.g.,~\cite{bach-sgd,rakhlin2012making,ghadimi2012optimal,ghadimi2013stochastic,curtis-large-scale_ml}. The classical convergence results are mostly \textit{in expectation}\footnote{Also commonly referred to as the \emph{mean-squared error}.}, characterizing the average performance across many runs of the algorithm. However, due to significant computational costs of a single run of an algorithm in many modern large-scale machine learning applications, it is often infeasible to perform multiple runs, e.g.,~\cite{harvey2019tight, davis2021low}. As such, the classical in expectation results do not fully capture the convergence behavior of these algorithms. More fine-grained results, like \emph{high-probability convergence}, characterizing the behaviour of an algorithm with respect to a single run, offer better guarantees when multiple runs are infeasible.

Another striking feature of existing works is assuming that the gradient noise is \textit{light-tailed} or has uniformly \textit{bounded variance}, e.g.,~\cite{rakhlin2012making, ghadimi2012optimal, ghadimi2013stochastic}. As we discuss next, this is a major limitation in many modern applications. It has been observed in applications like training attention models that SGD performs worse than adaptive methods, even after extensive hyperparameter tuning, e.g., \cite{simsekli2019tail, zhang2020adaptive}. In \cite{zhang2020adaptive}, it is shown that the gradient noise distribution during training of large attention models resembles a Levy $\alpha$-stable distribution with $\alpha < 2$, which has unbounded variance. The authors show that SGD fails to converge due to presence of large stochastic gradients. The clipped variant of SGD solves the problem and achieves \textit{optimal} convergence rate in expectation for smooth non-convex costs. Subsequently, clipped stochastic methods have been extensively analyzed in recent years to solve stochastic minimization, e.g., \cite{gorbunov2020stochastic,sadiev2023highprobability}. Along with addressing heavy-tailed noise, clipped SGD also helps address non-smoothness of the objective function, e.g.,~\cite{zhang2019gradient}, achieve differential privacy, e.g.,~\cite{NEURIPS2020_9ecff545, zhang2022clip_FL_icml, yang2022normalized_sgd} and robustness to malicious nodes in distributed learning, e.g.,~\cite{shuhua-clipping}. 

Despite its popularity, clipping is not the only nonlinear transformation of SGD employed in practice. Sign and quantized variants of SGD improve communication efficiency in distributed learning, e.g., \cite{alistarh2017qsgd, bernstein2018signsgd, gandikota2021vqsgd}. Sign SGD achieves performance on par with state-of-the-art adaptive methods, e.g.,~\cite{crawshaw2022general_signSGD}, and is robust to faulty/malicious users, e.g.,~\cite{bernstein2018signsgd_iclr}. Normalized SGD is empirically observed to accelerate neural network training, e.g.,~\cite{hazan2015beyond, you2019reducing, cutkosky20normalized_SGD}, facilitate private learning, e.g.,~\cite{das2021DP_normFedAvg, yang2022normalized_sgd}, and is effective in solving distributionally robust optimization, e.g.,~\cite{jin2021non-convex_DRO}. \cite{zhang2020adaptive} observed empirically that component-wise clipping converges faster than joint clipping, exhibiting a better dependence on problem dimension.

\begin{table*}[htp]
\caption{Comparison of streaming SGD methods. Lower-case $t$ indicates a streaming method, upper-case $T$ indicates a preset time horizon is needed, $\beta \in (0,1)$ is the failure probability, while $\widetilde{\mathcal{O}}$ hides logarithmic factors.}
\label{tab:comp}
\begin{adjustwidth}{-1in}{-1in} 
\begin{center}
\begin{threeparttable}
\begin{small}
\begin{sc}
\begin{tabular}{ccccc}
\toprule
\multicolumn{1}{c}{\rule{0pt}{2.5ex}Cost} & \multicolumn{1}{c}{Work} & \multicolumn{1}{c}{Nonlinearity} & \multicolumn{1}{c}{Noise} & \multicolumn{1}{c}{Rate}\\
\midrule
\multirow{4}{*}{Non-convex} & \cite{nguyen2023improved} & Clipping only & $\substack{\text{\scriptsize{bounded moment of}} \\ \text{\scriptsize{order }} \eta \in (1,2]}$ & $\widetilde{\mathcal{O}}\left(t^{\frac{\eta - 2\eta}{3\eta - 2}} \right)$ \\
& \cite{sadiev2023highprobability} & Clipping only & $\substack{\text{\scriptsize{bounded moment of}} \\ \text{\scriptsize{order }} \eta \in (1,2]}$ & $\mathcal{O}\left(T^{\frac{1 - \eta}{\eta}} \right)^{1}$ \\
& This paper & Component-wise & $\substack{\text{\scriptsize{bounded first moment}}\\ \text{\scriptsize{pdf symmetric,}} \\ \text{\scriptsize{positive around zero}}}$ & $\mathcal{O}\left(t^{\delta - 1}\right)^{2}$ \\
\midrule
\multicolumn{1}{c}{\multirow{5}{*}{Strongly convex}} & \cite{pmlr-v151-tsai22a} & Clipping only & $\substack{\text{\scriptsize{bounded growth}} \\ \text{\scriptsize{second moment}}^3}$ & $\mathcal{O}\left(t^{-1} \right)$ \\
\multicolumn{1}{c}{} & \cite{sadiev2023highprobability} & Clipping only & $\substack{\text{\scriptsize{bounded moment of}} \\ \text{\scriptsize{order }} \eta \in (1,2]}$ & $\mathcal{O}\left(T^{\frac{2(1-\eta)}{\eta}} \right)$ \\
\multicolumn{1}{c}{} & $\substack{\text{\normalsize{This paper - weighted}} \\ \text{\normalsize average of iterates}}$ &  Component-wise & $\substack{\text{\scriptsize{bounded first moment}}\\ \text{\scriptsize{pdf symmetric,}} \\ \text{\scriptsize{positive around zero}}}$  & $\mathcal{O}\left(t^{2(\delta - 1)}\right)^4$ \\
\multicolumn{1}{c}{} &  This paper - last iterate  &  Component-wise and joint & $\substack{\text{\scriptsize{bounded first moment}}\\ \text{\scriptsize{pdf symmetric,}} \\ \text{\scriptsize{positive around zero}}}$  & $\mathcal{O}\left(t^{-\zeta}\right)^5$ \\
\bottomrule
\end{tabular}
\end{sc}
\end{small}
\begin{tablenotes}
    \item[1] \scriptsize{The method in~\cite{sadiev2023highprobability} uses a preset time horizon $T$ to tune algorithm parameters, (e.g., step-size). As such, it is not a streaming algorithm per se, however, it can be adapted to the streaming setting with minor tweaks (e.g., time-varying step-size).}
    \item[2] \scriptsize{The parameter $\delta \in (\nicefrac{3}{4},1)$ is user-specified. As such, our rate can be made arbitrarily close to $\mathcal{O}\left(t^{-\nicefrac{1}{4}} \right)$, by setting $\delta = \frac{3}{4}+\epsilon$, for $\epsilon < \frac{1}{4}$ small.}
    \item[3] \scriptsize{The gradient noise $\bz$ at a point $\bx$ satisfies $\E\|\bz\|^2 \leq C + B\|\bx - \bx^\star\|^2$, where $\bx^\star$ is the solution to~\eqref{eq:problem} and $C,B>0$ are constants.}
    \item[4] \scriptsize{The parameter $\delta \in (\nicefrac{3}{4},1)$ is user-specified. As such, our rate can be made arbitrarily close to $\mathcal{O}\left(t^{-\nicefrac{1}{2}} \right)$, by setting $\delta = \frac{3}{4}+\epsilon$, for $\epsilon < \frac{1}{4}$ small.}
    \item[5] \scriptsize{The rate $\zeta \in (0,1)$ depends on the choice of nonlinearity, noise and problem related parameters, see Sections~\ref{sec:main}, \ref{sec:an-num} and Appendix~\ref{app:rate}. We provide examples of noise for which $\zeta > \nicefrac{2(\eta - 1)}{\eta}$, see Examples~\ref{example:1}-\ref{example:4} ahead.}
\end{tablenotes}
\end{threeparttable}
\end{center}
\vskip -0.1in
\end{adjustwidth}
\end{table*}

\paragraph{Literature review.} We now review the literature on high-probability convergence of SGD and its variants. Initial works on high-probability convergence of stochastic gradient methods considered light-tailed noise (see Definition~\ref{def:subgaus}) and include \cite{nemirovski2009robust, lan2012optimal,hazan2014beyond, harvey2019tight} for convex, and \cite{ghadimi2013stochastic,li2020high} for non-convex costs. Subsequent works \cite{gorbunov2020stochastic,gorbunov2021near,parletta2022high} generalized these results to noise with bounded variance. \cite{pmlr-v151-tsai22a} study the behaviour of clipped SGD under the assumption that the noise variance is bounded by iterate distance (see Table~\ref{tab:comp}), while \cite{li2022high,eldowa2023general} considered sub-Weibull noise. Recent works~\cite{liu2023high,eldowa2023general} remove restrictive assumptions, like bounded stochastic gradients and bounded domain. \cite{sadiev2023highprobability} show that even with bounded variance and smooth, strongly-convex functions, vanilla SGD cannot achieve a \textit{logarithmic} dependence on the failure probability $\beta \in (0,1)$\footnote{For high-probability convergence guarantees of type ``with probability at least $1 - \beta$''.}, implying that the complexity of achieving a high-probability bound for SGD is much worse than the complexity of converging in expectation. As such, nonlinear SGD is required to handle tails heavier than sub-Gaussian. Recent works consider a broad class of heavy-tailed noises with bounded moments of order $\eta \in (1,2]$, e.g., \cite{nguyen2023improved, nguyen2023high, sadiev2023highprobability,liu2023breaking}. \cite{nguyen2023improved,nguyen2023high} study high-probability convergence of clipped SGD for convex and non-convex minimization, \cite{sadiev2023highprobability} study clipped SGD for optimization and variational inequality problems, while \cite{liu2023breaking} study accelerated variants of clipped SGD for smooth costs. It is worth mentioning a recent work~\cite{gorbunov2023breaking}, which shows clipped SGD can achieve the optimal $\mathcal{O}\left(T^{-1}\right)$ rate for smooth strongly convex costs and a class of heavy-tailed noises with possibly unbounded first moments. However, they use a median-of-means gradient estimator, which requires storing multiple stochastic gradients prior to performing the update and it is not clear how such an approach can be extended to the streaming setting considered in this paper. 

The works closest to ours are~\cite{nguyen2023improved,sadiev2023highprobability} for non-convex and~\cite{pmlr-v151-tsai22a,sadiev2023highprobability} for strongly convex costs. For non-convex costs, the optimal rate $\widetilde{\mathcal{O}}\left(t^{\frac{2-2\eta}{3\eta-2}} \right)$ is achieved in~\cite{nguyen2023improved}. For the same costs, we consider a broad class of component-wise nonlinearities (e.g., sign, cclip and quantization) in the presence of noise with symmetric density and bounded first moment, achieving the rate $\mathcal{O}\left(t^{-\nicefrac{1}{4}+\epsilon}\right)$, for any $\epsilon > 0$. As such, our rate exponent \emph{is independent of noise or problem related parameters}, which is not the case with~\cite{nguyen2023improved,sadiev2023highprobability}, with our rate dominating that from~\cite{nguyen2023improved} whenever $\eta < \frac{6+8\epsilon}{5+12\epsilon}$.\footnote{This does not contradict the optimality of the rate achieved in~\cite{nguyen2023improved}, as their assumptions slightly differ from ours. Whereas~\cite{nguyen2023improved} require bounded noise moment of order $\eta \in (1,2]$, we require noise with symmetric density, but allow for bounded first moment only. As such, our work shows that additional structure in the noise leads to improved results, while allowing for relaxed moment conditions and heavier tails (see Examples~\ref{example:1}-\ref{example:3}).} Furthermore,~\cite{nguyen2023improved,sadiev2023highprobability} require knowledge of noise moment $\eta$ and other problem parameters to tune the step-size and clipping radius, whereas our analysis \emph{requires no knowledge of problem or noise related parameters}. For strongly convex costs,~\cite{pmlr-v151-tsai22a,sadiev2023highprobability} study the convergence of the last iterate for clipped SGD, with~\cite{sadiev2023highprobability} allowing for a more general class of noise, achieving the rate $\mathcal{O}\left(T^{\nicefrac{2(1-\eta)}{\eta}}\right)$. Compared to them, we consider a more general framework for nonlinear SGD (both component-wise and joint), establishing a rate $\mathcal{O}\left(t^{-\zeta} \right)$, for some $\zeta \in (0,1)$ that depends on the noise, nonlinearity and other problem parameters. We give examples of noise regimes where our rate is better than the one in~\cite{sadiev2023highprobability} (see Examples~\ref{example:1}-\ref{example:4}). Moreover, we demonstrate both analytically and numerically that $\zeta$ is informative for choosing the best nonlinearity for given problem and noise settings and that \emph{clipping is not always the best chocice of nonlinearity} (see Section~\ref{sec:an-num}), further highlighting the importance and usefulness of our general framework. The comparison is summed up in Table~\ref{tab:comp}. Finally, it is worth mentioning~\cite{jakovetic2023nonlinear}, who study the same general framework for nonlinear SGD and strongly convex costs, in expectation and almost sure sense. Our work differs in that we study high-probability convergence and allow for non-convex costs. The latter is achieved by providing a novel characterization of the interplay of the ``denoised'' nonlinear gradient and the true, noiseless gradient for component-wise nonlinearities (see Section~\ref{sec:main}). For strongly convex costs, naively combining the in expectation bound from~\cite{jakovetic2023nonlinear} and Markov inequality results in a sub-optimal $\nicefrac{1}{\beta}$ dependence on the failure probability $\beta$, with our work closing this gap, by showing the optimal $\log(\nicefrac{1}{\beta})$ dependence. 

\paragraph{Contributions.} Our contributions can be summarized as follows.
\begin{itemize}
    \item We provide a unified framework for studying convergence in high probability of nonlinear streaming SGD under heavy-tailed noise. In the proposed framework, the nonlinearity is treated in a black-box manner, subsuming many popular nonlinearities, like sign, normalization, cclip and quantization. \emph{To the best of our knowledge, we provide the first high-probability results under heavy-tailed noise for methods such as sign, quantized and component-wise clipped SGD}. 

    \item For non-convex costs and component-wise nonlinearities, we show a convergence rate of $\mathcal{O}\left(t^{-\nicefrac{1}{4} + \epsilon}\right)$, for any $\epsilon > 0$. The exponent in our rate is independent of noise and problem parameters, which is not the case for state-of-the-art rate in~\cite{nguyen2023improved}, with our rate dominating the state-of-the-art whenever the noise has bounded moments of order $\eta < \frac{6+8\epsilon}{5+12\epsilon}$. Additionally, our analysis requires no knowledge of problem parameters to tune the step-size, whereas the analysis in~\cite{nguyen2023improved} requires knowledge of noise moments and problem parameters to tune the step-size and clipping radius. 

    \item For strongly convex costs and component-wise nonlinearities we show convergence of the weighted average of iterates with the same rate $\mathcal{O}\left(t^{-\nicefrac{1}{2} + \epsilon}\right)$, with our rate dominating the state-of-the-art~\cite{sadiev2023highprobability}, whenever the noise has bounded moments of order $\eta < \frac{4}{3+2\epsilon}$. Next, we show convergence of the last iterate for a broader class of nonlinearities, with rate $\mathcal{O}\left(t^{-\zeta} \right)$, where $\zeta \in (0,1)$ depends on noise, nonlinearity and other problem parameters. As such, the exponent $\zeta$ is shown to be informative in choosing the best nonlinearity for the given problem and noise settings, both analytically and numerically. We provide examples of noise regimes in which our exponent $\zeta$ dominates the one instate-of-the-art~\cite{sadiev2023highprobability}.
    
    \item Compared to state-of-the-art~\cite{nguyen2023improved,sadiev2023highprobability}, who only consider clipping, require bounded noise moments of order $\eta \in (1,2]$ and vanishing rates as $\eta \rightarrow 1$, we consider a much broader class of nonlinearities under noise with symmetric density, while relaxing the moment condition and providing non-vanishing convergence rates even for noise with \emph{finite first moment only}. Moreover, for strongly convex costs, we provide analytical and numerical results that show \emph{clipping is not always the optimal choice of nonlinearity}, further reinforcing the importance of our general framework.
\end{itemize}

\paragraph{Paper organization.} The rest of the paper is organized as follows. Section~\ref{sec:framework} outlines the nonlinear streaming SGD framework. Section~\ref{sec:main} presents the main results of the paper. Section~\ref{sec:an-num} provides analytical and numerical results demonstrating that clipping is not always the optimal choice of nonlinearity. Section~\ref{sec:conclusion} concludes the paper. The Appendix presents some useful facts and results omitted from the main body. The remainder of this section introduces the notation.    

\paragraph{Notation.}
We denote the set of positive integers by $\N$, while $\N_0$ denotes the set of non-negative integers, i.e., $\N_0 \triangleq \N \, \cup \, \{0 \}$. We use $\R$, $\R_+$ and $\R^d$ to denote the sets of real numbers, non-negative real numbers and the $d$-dimensional real vector space, respectively. For $a \in \N$, $[a]$ denotes the set of integers up to and including $a$, i.e., $[a] = \{1,\ldots,a \}$. Regular and bold symbols denote scalars and vectors, respectively, i.e., $x \in \R$ and $\bx \in \R^d$. The Euclidean inner product is denoted by $\langle \cdot,\cdot\rangle$, while $\|\cdot\|$ denotes the induced norm.

\section{Proposed Framework}\label{sec:framework}

To solve~\eqref{eq:problem} in the streaming setting and in the presence of heavy-tailed noise, we use the \textit{nonlinear SGD} framework. The algorithm starts by choosing a deterministic initial model $\bx^{(0)} \in \mbb R^d$ and a nonlinear map $\boldsymbol{\Psi}:\mbb R^d \mapsto \mbb R^d$. In iteration $t = 0,1,\ldots$, the method performs as follows: a new sample $\omega^{(t)}$ is observed and the gradient of the loss $\ell$ at the current model $\bxt$ and sample $\omega^{(t)}$ is computed\footnote{Equivalently, we have access to a first-order oracle that directly streams gradients of $\ell$, instead of samples.}. Then, the model is updated as
\begin{equation}\label{eq:update}
    \bxtp = \bxt - \alpha_t\mathbf{\Psi}\left(\nabla \ell(\bxt;\omega^{(t)})\right),
\end{equation} where $\alpha_t > 0$ is the step-size at iteration $t$. The method is summed up in Algorithm~\ref{alg:nonlin-sgd}. We make the following assumption on the nonlinear map $\bPsi$.

\begin{algorithm}[!tb]
\caption{Nonlinear SGD on Streaming Data}
\label{alg:nonlin-sgd}
\begin{algorithmic}[1]
   \REQUIRE{Choice of nonlinearity $\bPsi: \R^d \mapsto \R^d$, model initialization $\bx^{0} \in \R^{d}$;}
   \AT{time t = 0,1,2,\ldots}:
        \STATE \hspace{1em}Observe a new sample $\omega^{(t)}$ and compute the gradient $\nabla \ell(\bxt;\omega^{(t)})$;  
        \STATE \hspace{1em}Update the model $\bxtp \leftarrow \bxt - \alpha_t\mathbf{\Psi}\left(\nabla \ell(\bxt;\omega^{(t)})\right)$;
\end{algorithmic}
\end{algorithm}

\begin{assump}\label{asmpt:nonlin}
The nonlinear map $\bPsi: \mbb R^d \mapsto \mbb R^d$ is either of the form $\bPsi(\bx) = \bPsi(x_1,\dots,x_d) = \lbr \calN_1(x_1), \dots, \calN_1(x_d) \rbr^\top$ or $\bPsi(\bx) = \bx\calN_2(\|\bx\|)$, where $\calN_1,\: \calN_2: \R \mapsto \R$ satisfy
\begin{enumerate}
    \item $\calN_1,\calN_2$ are continuous almost everywhere (with respect to the Lebesgue measure), with $\calN_1$ piece-wise differentiable, while the mapping $a \mapsto a\calN_2(a)$ is non-decreasing.
    \item $\calN_1$ is monotonically non-decreasing and odd function, while $\calN_2$ is non-increasing.
    \item $\calN_1$ is either discontinuous at zero, or strictly increasing on $(-c_1,c_1)$, for some $c_1 > 0$, with $\calN_2(a) > 0$, for any $a > 0$.
    \item $\calN_1$ and $\bx\calN_2(\|\bx\|)$ are uniformly bounded, i.e., $|\calN_1(x)| \leq C_1$ and $\|\bx\calN_2(\|\bx\|)\|\leq C_2$, for some $C_1,\: C_2 > 0$, and all $x \in \R$, $\bx \in \R^d$.
\end{enumerate}
\end{assump}

Note that the fourth property implies $\|\bPsi(\bx)\| \leq C$, where $C = C_1\sqrt{d}$ or $C = C_2$, depending on the form of nonlinearity. We will use the general bound $\|\bPsi(\bx) \| \leq C$ for ease of presentation, and specialize where appropriate. Assumption~\ref{asmpt:nonlin} is satisfied by a wide class of nonlinearities, including:
\begin{enumerate}[leftmargin=*]
    \item \emph{Sign}: $[\bPsi(\bx)]_i = \text{sign}(x_i), \: i = 1,\ldots,d$,
    \item \emph{Component-wise clip (cclip)}: $[\bPsi(\bx)]_i = x_i$, for $|x_i| \leq m$, and $[\bPsi(\bx)]_i = m\cdot\text{sign}(x_i)$, for $|x_i| > m$, $i = 1,\ldots,d$, for some constant $m>0$.
    \item \emph{Component-wise quantization}: for each $i = 1,\ldots,d$, let $[\bPsi(\bx)]_i = r_j$, for $x_i \in (q_j,q_{j+1}]$, with $j = 0,\ldots,J-1$ and $-\infty = q_0 < q_1 <\ldots < q_J = +\infty$, where $r_j$'s and $q_j$'s are chosen such that each component of $\bPsi$ is an odd function, and we have $\max_{j \in \{0,\ldots,J-1\}}|r_j| < R$, for $R > 0$.
    \item \emph{Normalization}: $\bPsi(\bx) = 
    \frac{\bx}{\norm{\bx}}$, with $\bPsi(\bx) = 0 \text{ if } \bx = \mathbf{0}$.
    \item \emph{Clipping}: $\bPsi(\bx) = \min\left\{1,\frac{M}{\|\bx\|}\right\}\bx$, for some constant $M>0$.
\end{enumerate}

\section{Main Results}\label{sec:main}

In this section we present the main results of the paper. Subsection~\ref{subsec:prelim} presents the preliminaries and assumptions, Subsection~\ref{subsec:theory-nonconv} presents the main theoretical results for general non-convex costs, while Subsection~\ref{subsec:theory-cvx} presents the main theoretical results for strongly convex costs. All the proofs can be found in Appendix~\ref{app:proofs}.

\subsection{Preliminaries}\label{subsec:prelim}

In this section we provide the preliminaries and assumptions used in the analysis. To begin with, we state the assumptions on the behaviour of the population loss $f$ used in the paper.

\begin{assump}\label{asmpt:L-smooth}
The population loss $f$ is bounded from below, has at least one stationary point and Lipschitz continuous gradients, i.e., $\inf_{\bx \in \R^d}f(\bx) > -\infty$, there exists a $\bx \in \R^d$ such that $\nabla f(\bx) = 0$, and for some $L > 0$ and every $\bx,\by \in \R^d$
\begin{equation*}
    \|\nabla f(\bx) - \nabla f(\by) \| \leq L\|\bx - \by\|.
\end{equation*}
\end{assump}

\begin{remark}
    Boundedness from below and Lipschitz continuous gradients are standard for non-convex costs, e.g.,~\cite{ghadimi2013stochastic,madden2020high,liu2023high}. Since the goal in non-convex optimization is to reach a stationary point, it is natural to assume at least one such point exists. 
\end{remark}

\begin{remark}
    It can be readily shown that Lipschitz continuous gradients imply, for any $\bx, \by \in \R^d$
    \begin{equation*}
        f(\by) \leq f(\bx) + \langle\nabla f(\bx),\by - \bx\rangle + \frac{L}{2} \|\bx - \by \|^2,
    \end{equation*} know as \emph{L-smoothness inequality/property}, see, e.g.,~\cite{nesterov-lectures_on_cvxopt,Wright_Recht_2022}.
\end{remark}

\begin{assump}\label{asmpt:cvx}
The population loss $f$ is strongly convex, i.e., for some $\mu > 0$ and every $\bx,\by \in \R^d$
\begin{equation*}
     f(\by) \geq f(\bx) + \langle\nabla f(\bx),\by - \bx\rangle + \frac{\mu}{2} \|\bx - \by \|^2. 
\end{equation*}
\end{assump}

\begin{remark}
     Combined with Assumption~\ref{asmpt:cvx}, it follows that $\mu \leq L$.  
\end{remark}

Denote the infimum of $f$ by $f^\star$, i.e., $f^\star \triangleq \inf_{\bx \in \R^d}f(\bx)$. Denote by $\mathcal{X} \subset \R^d$ the set of stationary points of $f$, i.e., $\mathcal{X} \triangleq \left\{\bx \in \R^d: \: \nabla f(\bx) = 0 \right\}$. Then, by Assumption~\ref{asmpt:L-smooth}, it follows that $\mathcal{X} \neq \emptyset$. Assumption~\ref{asmpt:cvx} implies the existence of a unique global minimizer, i.e., we have $f^\star = f(\bx^\star)$, for some unique $\bx^\star \in \R^d$. Equivalently, we have $\mathcal{X} = \left\{\bx^\star\right\}$. Next, rewrite the update~\eqref{eq:update} as
\begin{align}\label{eq:nonlin-sgd}
    \bxtp = \bxt - \alpha_t \boldsymbol{\Psi}(\G f(\bxt) + \bzt), 
\end{align} where $\bzt \triangleq \nabla \ell(\bxt;\omega^{(t)}) - \G f(\bxt)$ is the stochastic noise at iteration $t$. To simplify the notation, we use the shorthand $\bPsi^{(t)} \triangleq \boldsymbol{\Psi}(\nabla f(\bxt) + \bzt)$. We make the following assumption on the sequence of noise vectors $\{\bzt \}_{t \in \N_0}$.

\begin{assump}\label{asmpt:noise}
Noise vectors $\{\bzt \}_{t \in \N_0}$ are zero mean, independent, identically distributed and integrable, with symmetric probability density function (PDF) $p: \R^d \mapsto \R_+$, strictly positive in a neighborhood of zero, i.e., $p(\bz) > 0$, for all $\|\bz\| \leq B_0$ and some $B_0 > 0$.
\end{assump}

\begin{remark}
    Assumption~\ref{asmpt:noise} relaxes the noise moment condition made in~\cite{nguyen2023improved,sadiev2023highprobability}, at the expense of requiring a symmetric PDF, positive in a neighborhood of zero. PDF symmetry and positivity around zero are mild assumptions, satisfied by many noise distributions, such as Gaussian, as well as a broad class of heavy-tailed zero-mean $\alpha$-stable distributions, \cite{stable-distributions,csimcsekli2019heavy}, or the ones in Examples~\ref{example:1}-\ref{example:3} below. 
\end{remark}

We now give some examples of noise PDFs satisfying Assumption~\ref{asmpt:noise}.

\begin{example}\label{example:1}
    The noise PDF $p(\bz) = \rho(z_1)\times \ldots \times \rho(z_d)$, where $\rho(z) = \frac{\alpha - 1}{2(1 + |z|)^\alpha}$, for some $\alpha > 2$. It can be shown that the PDF only has finite $\eta$-th moments for $\eta < \alpha - 1$.
\end{example}

\begin{example}\label{example:2}
    The noise PDF $p(\bz) = \rho(z_1) \times \ldots \times \rho(z_d)$, where $\rho(z) = \frac{c}{(z^2 + 1)\log^2(|z| + 2)}$, with $c = \int \nicefrac{1}{[(z^2 + 1)\log^2(|z| + 2)]}dz$ being the normalizing constant. It can be shown that $p$ has a finite first moment, but for any $\eta \in (1,2]$, the $\eta$-th moments do not exist.
\end{example}

\begin{example}\label{example:3}
    The PDF $p: \R^d \mapsto \R_+$ with ``radial symmetry'', i.e., $p(\bz) = \rho(\|\bz\|)$, where $\rho: \R \mapsto \R_+$ is itself a PDF. If $\rho$ is the PDF from Example~\ref{example:2}, then $p$ inherits the properties of $\rho$, i.e., it does not have finite $\eta$-th moments, for any $\eta > 1$. 
\end{example}

Note that, while the noise from Example~\ref{example:1} satisfies the moment condition from~\cite{nguyen2023improved,sadiev2023highprobability}, the noise from Example~\ref{example:2} clearly does not. Next, define the function $\bPhi:\mbb R^d \mapsto \mbb R^d$, given by $\bPhi(\bx) \triangleq \mbe_{\bz} [\bPsi (\bx + \bz)] = \int \bPsi(\bx+\bz) p(\bz) d\bz$,\footnote{If $\bPsi$ is a component-wise nonlinearity, then $\bPhi$ is a vector with components $\phi_i(x_i) = \mbe_{z_i}[\calN_1(x_i + z_i)]$, where $\E_{z_i}$ is the marginal expectation with respect to the $i$-th noise component, $i \in [d]$ (see Lemma~\ref{lm:polyak-tsypkin} ahead).} where the expectation is taken with respect to the gradient noise at a random sample, i.e., $\bz \triangleq \nabla \ell(\bx;\omega) - \nabla f(\bx)$. We use the shorthand $\bPhit \triangleq \bPhi (\nabla f(\bxt))$. The vector $\bPhit$ can be seen as the denoised version of $\bPsi^{(t)}$. Using $\bPhit$, we can rewrite the update rule~\eqref{eq:nonlin-sgd} as
\begin{equation}\label{eq:nonlin-sgd2}
    \bxtp = \bxt - \alpha_t\bPhit + \alpha_t\bet,    
\end{equation} where $\bet \triangleq \boldsymbol{\Phi}^{(t)} - \boldsymbol{\Psi}^{(t)}$, represents the \emph{effective noise} term. As we show next, the effective noise is light-tailed, even though the original noise may not be. This allows us to establish exponential concentration inequalities and tight control of the behaviour of MGFs. Prior to that, we define the concept of \emph{sub-Gaussianity}, e.g., \cite{vershynin_2018,jin2019short}.

\begin{definition}\label{def:subgaus}
    A zero-mean random vector $\mathbf{v} \in \R^d$ is said to be sub-Gaussian, if there exists a constant $N > 0$, such that, for any $\bx \in \R^d$
    \begin{equation*}
        \E\left[ \exp\left(\langle\bx,\mathbf{v}\rangle\right)\right] \leq \exp\left(N\|\bx\|^2 \right).
    \end{equation*}
\end{definition}

It can be shown that any bounded random variable is sub-Gaussian, a fact used in the following sections (see Appendix~\ref{app:facts} for a proof). Define $\mathcal{F}_t$ to be the natural filtration, i.e., $\mathcal{F}_t \triangleq \sigma\left(\{\bz^{(0)},\ldots,\bz^{(t-1)}\}\right)$, with $\mathcal{F}_0 \triangleq \sigma(\{\emptyset,\Omega\})$ being the trivial $\sigma$-algebra. Next, we state some properties of the effective noise $\bet$.

\begin{lemma}\label{lm:error_component}
    Let Assumptions~\ref{asmpt:nonlin} and~\ref{asmpt:noise} hold. Then, the effective noise vectors $\{\bet\}_{t \in \mbb N_0}$ satisfy 
    \begin{enumerate}
        \item $\E[\bet\vert \: \mathcal{F}_t] = 0$  and  $\|\bet\| \leq 2C$,
        \item The effective noise is sub-Gaussian, i.e., for some $N > 0$ any $\bx \in \R^d$, we have $\E\left[\exp\left(\langle \bx, \bet \rangle \right) \: \vert \: \mathcal{F}_t \right] \leq \exp\left(N\|x\|^2\right)$
    \end{enumerate}
\end{lemma}

Note that the bound on the effective noise $\bet$ comes from the nonlinearity, therefore, the constant $N$ depends only on the nonlinearity, not the noise.

\subsection{Non-convex Costs}\label{subsec:theory-nonconv}

In this section we establish the convergence in high-probability for general non-convex functions and component-wise nonlinearities. This is made possible by establishing a novel characterization of the behaviour of $\bPhi(\bx)$ with respect to the original (noiseless) vector $\bx$. To begin with, we state a result from~\cite{polyak-adaptive-estimation}, that provides some basic properties of the mapping $\bPhi$ for component-wise nonlinearities.

\begin{lemma}\label{lm:polyak-tsypkin}
    Let Assumptions~\ref{asmpt:nonlin} and~\ref{asmpt:noise} hold, with the nonlinearity $\bPsi: \R^d \mapsto \R^d$ being component wise, i.e., of the form $\bPsi(\bx) = \begin{bmatrix} \calN_1(x_1),\ldots,\calN_1(x_d)\end{bmatrix}^\top$. Then, the function $\bPhi: \R^d \mapsto \R^d$ is of the form $\bPhi(\bx) = \begin{bmatrix} \phi_1(x_1),\ldots,\phi_d(x_d) \end{bmatrix}^\top$, where $\phi_i(x_i) = \E_{z_i}\left[\calN_1(x_i + z_i)\right]$ is the marginal expectation of the $i$-th noise component, $i \in [d]$, with the following properties:
    \begin{enumerate}
        \item $\phi_i$ is non-decreasing and odd, with $\phi_i(0) = 0$;
        \item $\phi_i$ is differentiable in zero, with $\phi_i^\prime(0) > 0$.
    \end{enumerate}
\end{lemma}

Define $\phi^\prime(0) \triangleq \min_{i \in [d]}\phi_i^\prime(0)$. We then have the following result on the interplay of the ``denoised'' nonlinearity $\bPhi(\bx)$ and $\bx$. 

\begin{lemma}\label{lm:huber}
    Let Assumptions~\ref{asmpt:nonlin} and~\ref{asmpt:noise} hold, with the nonlinearity $\bPsi: \R^d \mapsto \R^d$ being component wise, i.e., of the form $\bPsi(\bx) = \begin{bmatrix} \calN_1(x_1),\ldots,\calN_1(x_d)\end{bmatrix}^\top$. Then, for any $\bx \in \R^d$, we have
    \begin{equation*}
        \langle \bPhi(\bx), \bx \rangle \geq \frac{\phi^\prime(0)}{2}\min\{\nicefrac{\xi\|\bx\|}{ \sqrt{d}},\nicefrac{\|\bx\|^2}{d}\}, 
    \end{equation*} where $\xi > 0$ is a global constant depending only on the choice of nonlinearity and noise.
\end{lemma}

Lemma~\ref{lm:huber} provides a novel characterization of the inner product of the ``denoised'' nonlinearity $\bPhi$ at vector $\bx$ and vector $\bx$ itself. For any $k = 0,\ldots, t-1$, define $\widetilde{\alpha}_k \triangleq \frac{\alpha_k}{\sum_{j = 0}^{t-1}\alpha_j}$, so that $\sum_{k = 0}^{t-1}\widetilde{\alpha}_k = 1$, and for any $\bx \in \R^d$ define $D_{\mathcal{X}}(\bx^{(0)}) \triangleq \inf_{\bx \in \mathcal{X}}\|\bx^{(0)} - \bx\|^2$ to be the set distance function. We are now ready to state our high-probability convergence bound of component-wise nonlinear SGD for non-convex costs.

\begin{theorem}\label{thm:non-conv}
    Let Assumptions~\ref{asmpt:nonlin}, \ref{asmpt:L-smooth} and~\ref{asmpt:noise} hold, with the nonlinearity $\bPsi: \R^d \mapsto \R^d$ being component-wise, i.e., of the form $\bPsi(\bx) = \begin{bmatrix} \calN_1(x_1),\ldots,\calN_1(x_d)\end{bmatrix}^\top$. Let $\{\bxt\}_{t \in \N_0}$ be the sequence generated by~\eqref{eq:nonlin-sgd}, with step-size $\alpha_t = \frac{a}{(t + 2)^\delta}$, for any $\delta \in (\nicefrac{3}{4},1)$ and $a > 0$. Then, for any $t \in \N_0$, and any $\beta \in (0,1)$, with probability at least $1 - \beta$, it holds that
    \begin{equation*}
        \sum_{k = 0}^{t-1}\widetilde{\alpha}_k\min\{\nicefrac{\xi\|\nabla f(\bxk)\|}{ \sqrt{d}},\nicefrac{\|\nabla f(\bxk)\|^2}{d}\} \leq \frac{R(a,\beta,\delta)}{(t+2)^{1-\delta} - 2^{1-\delta}},
    \end{equation*} where $R(a,\beta,\delta) \triangleq \frac{2(1-\delta)}{\phi^\prime(0)}\left[\nicefrac{\left(f(\bx^{(0)}) - f^\star + \log(\nicefrac{1}{\beta})\right)}{a} + \frac{a(\nicefrac{dLC_1^2}{2} + 2NL^2D_{\mathcal{X}}(\bx^{(0)}))}{(2\delta-1)} + \frac{2a^3dNC_1^2L^2}{(1-\delta)^2(4\delta-3)}\right]$.
\end{theorem}

Some remarks are now in order. 

\begin{remark}\label{rmk:min}
    The bound from Theorem~\ref{thm:non-conv} can be used to provide a bound on the best iterate, i.e., on the quantity $\min_{0 \leq k \leq t-1}\|\nabla f(\bxk)\|$. In particular, one can show that Theorem~\ref{thm:non-conv} implies
    \begin{equation*}
        \min_{0 \leq k \leq t-1}\|\nabla f(\bxk)\| = \mathcal{O}\left(\sqrt{\frac{dR(a,\beta,\delta)}{(t+2)^{1-\delta} - 2^{1-\delta}}}\right), 
    \end{equation*} i.e., in order to ensure we reach an $\epsilon$-stationary point\footnote{A point $\bx \in \R^d$ such that $\|\nabla f(\bx)\| \leq \epsilon$.} (with high probability), we require at least $\mathcal{O}\left(\epsilon^{-\frac{2}{1 - \delta}} \right)$ iterations. Derivations connecting the bound from Theorem~\ref{thm:non-conv} to the best iterate can be found in Appendix~\ref{app:further}. 
\end{remark}

\begin{remark}\label{rmk:rate+eps}
    The rate achieved by our analysis is of the order $\mathcal{O}\left(t^{\delta - 1}\right)$, where $\delta \in (\nicefrac{3}{4},1)$ is user-specified. As such, the exponent in our convergence rate is \emph{independent of any problem related parameters}, like the choice of nonlinearity, noise, etc. This is not the case with state-of-the-art methods, e.g.,~\cite{nguyen2023improved,sadiev2023highprobability}, whose rate exponent explicitly depends on noise moment $\eta \in (1,2]$ and vanishes as $\eta \rightarrow 1$. By choosing $\delta = \nicefrac{3}{4} + \epsilon$, for some $\epsilon \in (0,\nicefrac{1}{4})$, it follows that our method can achieve the rate $\mathcal{O}\left(t^{-\nicefrac{1}{4} + \epsilon} \right)$, which can be made arbitrarily close to $t^{-\nicefrac{1}{4}}$, for small $\epsilon$. In this case $R(a,\beta,\nicefrac{3}{4}+\epsilon) = \mathcal{O}\left(\epsilon^{-1} \right)$, with the convergence rate being $\mathcal{O}\left(\epsilon^{-1}t^{-\nicefrac{1}{4}+\epsilon}\right)$. We can see an inherent trade-off with respect to $\epsilon$, where smaller value of $\epsilon$ results in a convergence rate closer to $t^{-\nicefrac{1}{4}}$, at the cost of a larger constant factor (i.e., $\epsilon^{-1}$). Compared to the convergence rate $\mathcal{O}\left(t^{\frac{2-2\eta}{3\eta - 2}}\right)$, achieved in~\cite{nguyen2023improved}, our rate is better if $\frac{2\eta - 2}{3\eta - 2} < \frac{1}{4} - \epsilon$, i.e., whenever $\eta < \frac{6+8\epsilon}{5+12\epsilon}$.
\end{remark}

\begin{remark}
    The constant $R(a,\beta,\delta)$ depends on multiple problem related quantities, such as the problem dimension $d$, the optimality gap $f(\bx^{(0)})-f^\star$, the distance of the initial model from the set of stationary points $D_{\mathcal{X}}(\bx^{(0)})$, choice of nonlinearity and noise through $C_1,N$ and $\phi^\prime(0)$. The step-size parameter $a > 0$ offers an inherent trade-off, as choosing $a \approx 0$ alleviates the dependence of $R$ on multiple problem parameters (e.g., initial model from set of stationary points, or the effect of small $\epsilon$ discussed in Remark~\ref{rmk:rate+eps}), making it approximately $R(a,\beta,\delta) \approx \frac{2(1-\delta)}{\phi^\prime(0)}\nicefrac{\left(f(\bx^{(0)}) - f^\star + \log(\nicefrac{1}{\beta})\right)}{a}$, while simultaneously resulting in small step-sizes and larger constant factor (of the order $\nicefrac{1}{a}$), slowing convergence down. 
\end{remark}

\begin{proof}[Proof of Theorem~\ref{thm:non-conv}]
    For ease of notation, let $Z(\|\nabla f(\bxt)\|) \triangleq \min\{\nicefrac{\xi\|\nabla f(\bxt) \|}{\sqrt{d}},\nicefrac{\|\nabla f(\bxt)\|^2}{d}\}$. Applying the $L$-smoothness property of $f$ and the update rule~\eqref{eq:nonlin-sgd2}, to get
    \begin{align*}
        f(\bxtp) &\leq f(\bxt) - \alpha_t\langle \nabla f(\bxt),\bPhit - \bet \rangle + \frac{\alpha_t^2L}{2}\|\bPsi^{(t)}\|^2 \\ &\leq f(\bxt) - \frac{\alpha_t\phi^\prime(0)}{2}Z(\|\nabla f(\bxt)\|) + \alpha_t\langle \nabla f(\bxt),\bet\rangle + \frac{\alpha_t^2dLC_1^2}{2}, 
    \end{align*} where the second inequality follows from Lemma~\ref{lm:huber} and Assumption~\ref{asmpt:nonlin}. Rearranging and summing up the first $t$ terms, we get
    \begin{equation}\label{eq:5}
        \frac{\phi^\prime(0)}{2}\sum_{k = 0}^{t-1}\alpha_kZ(\|\nabla f(\bxk)\|) \leq \underbrace{f(\bx^{(0)}) - f^\star + \frac{dLC_1^2}{2}\sum_{k = 1}^t\alpha_k^2}_{\eqqcolon B_1} + \underbrace{\sum_{k = 1}^t\alpha_k\langle \nabla f(\bxk), \bek \rangle}_{\eqqcolon B_2}. 
    \end{equation} Denote the right-hand side of~\eqref{eq:5} by $Z_t$, i.e., $Z_t \triangleq \frac{\phi^\prime(0)}{2}\sum_{k = 1}^t\alpha_kZ(\|\nabla f(\bxk)\|)$ and note that $B_1$ is independent of the noise, i.e., is a deterministic quantity. We then have
    \begin{equation*}
        \E\left[\exp(Z_t) \right] \stackrel{\eqref{eq:5}}{\leq} \E\left[\exp\left(B_1 + B_2\right) \right] = \exp\left(B_1 \right)\E\left[\exp(B_2)\right]. 
    \end{equation*} We now bound $\E[\exp(B_2)]$. Denote by $\E_t[\cdot] \triangleq \E[ \cdot \: \vert \: \mathcal{F}_t]$ the expectation conditioned on history up to time $t$. We then have
    \begin{align*}
        \E[\exp(B_2)] &= \E\left[\exp\left(\sum_{k = 0}^{t-1}\alpha_k\langle \nabla f(\bxk), \bek \rangle\right)\right] \\ &= \E\left[\exp\left(\sum_{k = 0}^{t-2}\alpha_k\langle \nabla f(\bxk), \bek \rangle \right)\E_{t-1}\left[\exp(\alpha_{t-1}\langle \nabla f(\bx^{(t-1)}),\be^{(t-1)} \rangle) \right] \right] \\ &\leq \E\left[\exp\left(\sum_{k = 0}^{t-2}\alpha_k\langle \nabla f(\bxk), \bek \rangle \right) \exp\left(N\alpha_{t-1}^2\|\nabla f(\bx^{(t-1)})\|^2 \right) \right] \\ &\leq \ldots \leq \E\left[\exp\left(N\sum_{k = 0}^{t-1}\alpha_k^2\|\nabla f(\bxk)\|^2 \right) \right],
    \end{align*} where we repeatedly use Lemma~\ref{lm:error_component}. Next, consider $\|\nabla f(\bxk)\|$, for any $k \geq 0$. Define $A_t \triangleq \sum_{k = 0}^{t-1}\alpha_k$ and use $L$-smoothness, to get
    \begin{align*}
        \|\nabla f(\bxk)\| \leq L\|\bxk - x^\star\| &= L\|\bx^{(k-1)} - \alpha_{k-1}\bPsi^{(k-1)} - \bx^\star\| \leq L\left(\|\bx^{(k-1)} - \bx^\star\| +  \alpha_{k-1}\sqrt{d}C_1\right) \\ &\leq \ldots \leq L\left(\|\bx^0 - \bx^\star\| + \sqrt{d}C_1\sum_{s = 0}^{k-1}\alpha_s\right) = L\left(\|\bx^{(0)} - \bx^\star\| + \sqrt{d}C_1A_k\right),
    \end{align*} where we recall that $\bx^\star \in \mathcal{X} = \left\{\bx \in \R^d: \|\nabla f(\bx) \| = 0 \right\}$ is any stationary point of $f$. Therefore, we have
    \begin{equation*}
        \E\left[\exp(B_2) \right] \leq \exp\left(2NL^2D_{\mathcal{X}}(\bx^{(0)})\sum_{k = 0}^{t-1}\alpha_k^2 + 2dNC_1^2L^2\sum_{k = 1}^t\alpha_k^2A_k^2 \right),
    \end{equation*} where $D_{\mathcal{X}}(\bx^{(0)}) = \inf_{\bx \in \mathcal{X}}\|\bx^{(0)} - \bx\|^2$ is the distance of the initial model estimate from the set of stationary points. Combining everything, we get
    \begin{equation*}
        \E\left[\exp(Z_t) \right] \leq \exp\left(f(\bx^{(0)}) - f^\star + \left(\nicefrac{dLC_1^2}{2} + 2NL^2D_{\mathcal{X}}(\bx^{(0)})\right)\sum_{k = 1}^t\alpha_k^2 + 2dNC_1^2L^2\sum_{k = 1}^t\alpha_k^2A_{k}^2 \right).
    \end{equation*} Define $G_t \triangleq f(\bx^{(0)}) - f^\star + \left(\nicefrac{dLC_1}{2} + 2NL^2D_{\mathcal{X}}(\bx^{(0)})\right)\sum_{k = 1}^t\alpha_k^2 + 2dNC_1^2L^2\sum_{k = 1}^t\alpha_k^2A_k^2$. Using Markov's inequality, it then follows that, for any $\epsilon > 0$
    \begin{equation*}
        \mathbb{P}(Z_t > \epsilon) \leq \exp(-\epsilon)\E[\exp(Z_t)] \leq \exp(-\epsilon + G_t) \iff \mathbb{P}(Z_t > \epsilon + G_t) \leq \exp(-\epsilon).
    \end{equation*} Finally, for any $\beta \in (0,1)$, with probability at least $1 - \beta$, we have
    \begin{equation}\label{eq:6}
        Z_t \leq \log(\nicefrac{1}{\beta}) + G_t \iff A_t^{-1}Z_t \leq A_t^{-1}\left(\log(\nicefrac{1}{\beta}) + G_t \right).
    \end{equation} Consider the step-size schedule $\alpha_t = \frac{a}{(t + 2)^\delta}$, for $\delta \in (\nicefrac{3}{4},1)$. Using upper and lower Darboux sums, we get
    \begin{align*}
        \frac{a}{1-\delta}((t+2)^{1-\delta} - 2^{1-\delta})&\leq A_t \leq \frac{a}{1-\delta}((t+1)^{1-\delta} - 1), \\
        \frac{a^2}{2\delta - 1}(2^{1-2\delta} - (t+2)^{1-2\delta})&\leq \sum_{k = 1}^t\alpha_k^2 \leq \frac{a^2}{2\delta - 1}(1 - (t+1)^{1-2\delta}).
    \end{align*} Plugging in~\eqref{eq:6}, we then get, with probability at least $1 - \beta$
    \begin{align*}
        \frac{\phi^\prime(0)}{2}\sum_{k = 0}^{t-1}&\widetilde{\alpha}_kZ(\|\nabla f(x_k)\|) \leq \frac{(1-\delta)\left(f(\bx^{(0)}) - f^\star + \log(\nicefrac{1}{\beta})\right)}{a((t+2)^{1-\delta}-2^{1-\delta})} \\ &+ \frac{a(1-\delta)(\nicefrac{dLC_1^2}{2} + 2NL^2D_{\mathcal{X}}(\bx^{(0)}))}{(2\delta-1)((t+2)^{1-\delta} - 2^{1-\delta})} + \frac{2a^3dNC_1^2L^2\sum_{k = 0}^{t-1}(k+2)^{2-4\delta}}{(1-\delta)((t+2)^{1-\delta} - 2^{1 - \delta})}.
    \end{align*} Using the upper Darboux sum once more, we have
    \begin{equation*}
        \sum_{k = 0}^{t-1}(k+2)^{2-4\delta} \leq \int_{1}^{t+1}k^{2-4\delta}dk \leq \frac{1}{4\delta-3},
    \end{equation*} therefore, combining everything, we finally get
    \begin{equation*}
        \sum_{k = 0}^{t-1}\widetilde{\alpha}_kZ(\|\nabla f(x_k)\|) \leq \frac{R(a,\beta,\delta)}{(t+2)^{1-\delta} - 2^{1-\delta}},
    \end{equation*} where $R(a,\beta,\delta) = \frac{2(1-\delta)}{\phi^\prime(0)}\left[\nicefrac{\left(f(\bx^{(0)}) - f^\star + \log(\nicefrac{1}{\beta})\right)}{a} + \frac{a(\nicefrac{dLC_1^2}{2} + 2NL^2D_{\mathcal{X}}(\bx^{(0)}))}{(2\delta-1)} + \frac{2a^3dNC_1^2L^2}{(1-\delta)^2(4\delta-3)}\right]$.
\end{proof}

\subsection{Strongly Convex Costs}\label{subsec:theory-cvx}

In this section we establish the convergence in high-probability for strongly convex functions. First, recall the definition of the Huber loss $H_{\lambda}: \R \mapsto [0,\infty)$, parametrized by $\lambda > 0$, e.g.,~\cite{huber_loss}, which is given by 
\begin{equation*}
    H_{\lambda}(x) \triangleq \begin{cases}
        \frac{1}{2}x^2, & |x| \leq \lambda, \\
        \lambda|x| - \frac{\lambda^2}{2}, & |x| > \lambda.
    \end{cases}
\end{equation*}
By the definition of Huber loss, it is not hard to see that it is a convex, non-decreasing function on $[0,\infty)$. Moreover, it follows that
\begin{equation}\label{eq:7}
    \min\{\nicefrac{\xi\|\nabla f(\bxk)\|}{\sqrt{d}},\nicefrac{\|\nabla f(\bxk)\|^2}{d} \} \triangleq Z(\|\nabla f(\bxk)\|) \geq H_{\xi}(\|\nabla f(\bxk)\| / \sqrt{d}) = \frac{1}{d}H_{\xi\sqrt{d}}(\|\nabla f(\bxk)\|),
\end{equation} where the last inequality follows by noticing that $H_{\xi}(x/d) = \frac{1}{d}H_{\xi\sqrt{d}}(x)$. Next, recall that $\mu$-strongly convex costs satisfy the \emph{gradient domination property}, i.e., $\mu\|\bxk - \bx^\star\| \leq \|\nabla f(\bx) \|$, for any $\bx \in \R^d$, e.g.,~\cite{nesterov-lectures_on_cvxopt}. Combining~\eqref{eq:7} with the gradient domination property, we get 
\begin{equation*}
    \sum_{k = 0}^{t-1}\widetilde{\alpha}_k Z(\|\nabla f(\bxk)\|) \geq \frac{1}{d}\sum_{k = 0}^{t-1}\widetilde{\alpha}_kH_{\xi\sqrt{d}}(\mu\|\bxk - \bx^\star\|) \geq \frac{\mu^2}{d}H_{\xi\sqrt{d}/\mu}(\|\widehat{\bx}^{(t)} - \bx^\star\|),
\end{equation*} where $\widehat{\bx}^{(t)} \triangleq \sum_{k = 0}^{t-1}\widetilde{\alpha}_k\bxk$ is the weighted average of the first $t$ iterates, the first inequality follows from~\eqref{eq:7}, the gradient domination property and the fact that $H$ is non-decreasing, while the second property follows from the fact that $H$ is convex and non-decreasing $H$ and applying Jensen's inequality twice. Therefore, we have just shown the following.

\begin{theorem}\label{thm:polyak}
    Let Assumptions~\ref{asmpt:nonlin}-\ref{asmpt:noise} hold, with the nonlinearity $\bPsi: \R^d \mapsto \R^d$ being component-wise, i.e., of the form $\bPsi(\bx) = \begin{bmatrix} \calN_1(x_1),\ldots,\calN_1(x_d)\end{bmatrix}^\top$. Let $\{\bxt\}_{t \in \N_0}$ be the sequence generated by~\eqref{eq:nonlin-sgd}, with step-size $\alpha_t = \frac{a}{(t + 2)^\delta}$, for any $\delta \in (\nicefrac{3}{4},1)$ and $a > 0$. Then, for any $t \in \N_0$, and any $\beta \in (0,1)$, with probability at least $1 - \beta$, it holds that
    \begin{equation*}
         H_{\xi\sqrt{d}/\mu}(\|\widehat{\bx}^{(t)} - \bx^\star\|) \leq \frac{\widetilde{R}(a,\beta,\delta)}{(t+2)^{1-\delta} - 2^{1-\delta}},
    \end{equation*} where $\widetilde{R}(a,\beta,\delta) \triangleq \frac{2d(1-\delta)}{\mu^2\phi^\prime(0)}\left[\nicefrac{\left(f(\bx^{(0)}) - f^\star + \log(\nicefrac{1}{\beta})\right)}{a} + \frac{a(\nicefrac{dLC_1^2}{2} + 2NL^2D_{\mathcal{X}}(\bx^{(0)}))}{(2\delta-1)} + \frac{2a^3dNC_1^2L^2}{(1-\delta)^2(4\delta-3)}\right]$.
\end{theorem}

The quantity $\widehat{\bx}^{(t)}$ can be seen as a generalized Polyak-Ruppert estimator, e.g.,~\cite{ruppert,polyak,polyak-ruppert}. Similarly to Remark~\ref{rmk:min}, it can be shown that the bound from Theorem~\ref{thm:polyak} implies
\begin{equation*}
    \|\widehat{\bx}^{(t)} - \bx^\star \|^2 = \mathcal{O}\left(t^{-2(1-\delta)} \right),
\end{equation*} giving a rate $\mathcal{O}\left(t^{-\nicefrac{1}{2} + \epsilon} \right)$, for any $\epsilon < \nicefrac{1}{2}$, with an \emph{exponent independent of noise and problem parameters} (see Appendix~\ref{app:further} for details). Our rate improves on the state-of-the-art rate from~\cite{sadiev2023highprobability}, whenever $\eta < \frac{4}{3+2\epsilon}$. Note that Theorem~\ref{thm:polyak} provides a bound on the weighted average of past iterates and component-wise nonlinearities. However, for strongly convex functions it is of interest to characterize the convergence guarantees of the last iterate, e.g.,~\cite{harvey2019tight,pmlr-v151-tsai22a,sadiev2023highprobability}. Moreover, we would like to establish high-probability convergence guarantees for a wider class of nonlinearities, including joint ones, like clipping and normalization. To that end, we first characterize the behaviour of the mapping $\bPhi$ in the general nonlinearity case.  

\begin{lemma}\label{lm:key}
    Let Assumptions~\ref{asmpt:nonlin}-\ref{asmpt:noise} hold and $\{\bxt\}_{t \in \N_0}$ be the sequence generated by~\eqref{eq:nonlin-sgd}, with step-size $\alpha_t = \frac{a}{(t + 2)^\delta}$, for any $\delta \in (0.5,1)$, $a > 0$. Then, for some $\gamma = \gamma(a) > 0$ and any $t \in \N_0$
    \begin{equation*}
        \langle \bPhit, \nabla f(\bxt) \rangle \geq \gamma(t+2)^{\delta - 1}\|\nabla f(\bxt)\|^2.
    \end{equation*}
\end{lemma}

We are now ready to state the main result.

\begin{theorem}
\label{theorem:main}
    Suppose Assumptions~\ref{asmpt:nonlin}-\ref{asmpt:noise} hold and $\{\bxt\}_{t \in \N_0}$ is the sequence generated by~\eqref{eq:nonlin-sgd}, with step-size $\alpha_t = \frac{a}{(t + 2)^\delta}$, for any $\delta \in (0.5,1)$ and $a > 0$. Then, for any $t \in \N_0$, and any $\beta \in (0,1)$, with probability at least $1 - \beta$, it holds that
    \begin{equation*}
        \|\bxtp - \bx^\star\|^2 \leq \frac{2\log\left(\nicefrac{e}{\beta}\right)}{\mu B(t+2)^\zeta},
    \end{equation*} 
    where $\zeta = \min\Big\{2\delta - 1, \nicefrac{a\mu\gamma}{2} \Big\}$, $B = \min\left\{\frac{1}{(f(\bx^{(0)}) - f^\star)}, \frac{\mu\gamma}{aL(2N + \nicefrac{C^2}{2})} \right\}$.
\end{theorem}

We specialize the value of $\gamma = \gamma(a)$ for different nonlinearities and discuss its impact on the rate in Appendix~\ref{app:rate}. The value of $\zeta$ can be explicitly calculated for specific choices of nonlinearities and noise. We now give some examples.

\begin{example}\label{example:4}
    For the noise from Example~\ref{example:1} and sign nonlinearity, it can be shown that $\zeta \approx \min\big\{2\delta - 1,\frac{\mu}{L}\frac{1-\delta}{\sqrt{d}}\frac{\alpha-1}{\alpha} \big\}$, see~\cite{jakovetic2023nonlinear}. For the same noise and cclip, it can be shown that $\zeta \approx \min\left\{2\delta-1,\frac{\mu}{L\sqrt{d}}\frac{(1-\delta)(m-1)(1-(m+1)^{-\alpha})}{m} \right\}$. On the other hand, the rate from~\cite{sadiev2023highprobability} for gradient clipping, adapted to the same noise, is $\nicefrac{2(r - 1)}{r}$, where $r \leq \min\{\alpha - 1,2 \}$. While $\nicefrac{2(r - 1)}{r} > \zeta$ for $\alpha$ above a certain threshold, as $\alpha \rightarrow 2$, we have $\nicefrac{2(r - 1)}{r} \rightarrow 0$, while our rate $\zeta$ stays strictly positive and bounded away from zero for both sign and cclip. Therefore, for $\alpha$ sufficiently close to $2$, we have $\zeta > \frac{2(r-1)}{r}$, i.e., our rate is better than the one in state-of-the-art~\cite{sadiev2023highprobability}.
\end{example}

\begin{proof}[Proof of Theorem~\ref{theorem:main}]
     Using $L$-smoothness of $f$, the update rule~\eqref{eq:nonlin-sgd2} and Lemma~\ref{lm:key}, we have
    \begin{align*}
        f (\bxtp) &\leq f(\bxt) - \alpha_t\langle \nabla f(\bxt), \bPhi^{(t)} - \bet \rangle + \mfrac{\alpha_t^2L}{2}\|\boldsymbol{\Psi}^{(t)}\|^2 \nn 
        \\ &\leq f(\bxt) - \mfrac{a\gamma\|\nabla f(\bxt)\|^2}{(t + 2)} + \mfrac{a\langle \nabla f(\bxt),\bet \rangle}{(t + 2)^\delta} + \mfrac{a^2LC^2}{2(t + 2)^{2\delta}}.
    \end{align*} Subtracting $f^\star$ from both sides of the inequality, defining $F^{(t)} = f(\bxt) - f^\star$ and using $\mu$-strong convexity of $f$, we get 
    \begin{equation}
        F^{(t+1)} \leq \left(1 - \mfrac{2\mu a\gamma}{t+2} \right)F^{(t)} + \mfrac{a\langle \nabla f(\bxt),\bet \rangle}{(t + 2)^\delta} + \mfrac{a^2LC^2}{2(t + 2)^{2\delta}}. \label{eq_proof:thm:nonL_cw_4}
    \end{equation}
    Let $\zeta = \min\left\{2\delta - 1, \nicefrac{a\gamma\mu}{2} \right\}$. Defining $Y^{(t + 1)} \triangleq (t + 2)^\zeta F^{(t+1)} = (t + 2)^\zeta(f(\bxtp) - f^\star)$, from \eqref{eq_proof:thm:nonL_cw_4} we get
    \begin{equation}
        Y^{(t + 1)} \leq a_tY^{(t)} + b_t\langle \nabla f(\bxt),\bet \rangle  + c_tV, \label{eq_proof:thm:nonL_cw_5}
    \end{equation} 
    where $a_t = \left(1 - \frac{2\mu a\gamma}{t+2} \right)\left(\frac{t + 2}{t + 1}\right)^\zeta$, $b_t = \frac{a}{(t + 2)^{\delta - \zeta}}$, $c_t = \frac{a^2}{(t + 2)^{2\delta - \zeta}}$ and $V = \frac{LC^2}{2}$. Denote the MGF of $Y^{(t)}$ conditioned on $\mathcal{F}_t$ as $M_{t+1\vert t}(\nu) = \E\left[\exp\left(\nu Y^{(t+1)}\right)\vert \mathcal{F}_t\right]$. We then have, for any $\nu \geq 0$
    \begin{align}
        M_{t+1\vert t}(\nu) \nn &\overset{(a)}{\leq} \E\left[\exp\left( \nu (a_tY^{(t)} + b_t\langle \bet,\nabla f(\bxt) \rangle  + c_tV ) \right) \big| \: \mathcal{F}_t \right] \nn \\ 
        &\overset{(b)}{\leq} \exp (\nu a_tY^{(t)} + \nu c_tV ) \mathbb{E}\left[\exp (\nu b_t\langle \bet, \nabla f(\bxt)\rangle ) \big| \: \mathcal{F}_t \right] \nn \\ 
        & \overset{(c)}{\leq} \exp\left(\nu a_tY^{(t)} + \nu c_tV + \nu^2b^2_tN\|\nabla f(\bxt)\|^2\right) \nn \\
        & \overset{(d)}{\leq} \exp\left(\nu a_tY^{(t)} + \nu c_tV + 2\nu^2b_t^{\prime 2}LNY^{(t)}\right), \label{eq_proof:thm:nonL_cw_6}
    \end{align} 
    where $(a)$ follows from \eqref{eq_proof:thm:nonL_cw_5}, $(b)$ follows from the fact that $Y^{(t)}$ is $\mathcal{F}_t$ measurable, $(c)$ follows from Lemma~\ref{lm:error_component}, in $(d)$ we use $\|\nabla f(\bx)\|^2 \leq 2L(f(\bx) - f^\star)$ and define $b^\prime_t = a\frac{(t+1)^\frac{-\zeta}{2}}{(t+2)^{\delta - \zeta}}$, so that $b_t = (t+1)^\frac{\zeta}{2}b^\prime_t$. For the choice $0 \leq \nu \leq B$, for some $B > 0$ (to be specified later), we get
    \begin{equation*}
        M_{t+1\vert t}(\nu) \leq \exp\left(\nu(a_t + 2b_t^{\prime 2}LNB)Y^{(t)}\right)\exp\left(\nu c_tV\right).
    \end{equation*} 
    Taking the full expectation, we get
    \begin{equation}
    \label{eq:induction_step}
        M_{t+1}(\nu) \leq M_t((a_t + 2b_t^{\prime2}LNB)\nu)\exp(\nu c_tV).
    \end{equation} Similarly to the approach in~\cite{harvey2019tight}, we now want to show that $M_t(\nu) \leq e^{\frac{\nu}{B}}$, for any $0 \leq \nu \leq B$ and any $t \geq 0$. We proceed by induction. For $t=0$, we have
    \begin{equation*}
        M_0(\nu) = \exp(\nu Y^{(0)}) = \exp\left(\nu (f(\bx^{(0)}) - f^\star) \right),
    \end{equation*} 
    where we simply used the definition of $Y^{(t)}$ and the fact that it is deterministic for $t = 0$. Choosing $B \leq (f(\bx^{(0)}) - f^\star)^{-1}$ ensures that $M_0(\nu) \leq e^{\frac{\nu}{B}}$. Next, assume that for some $t \geq 1$ it holds that $M_t(\nu) \leq e^{\frac{\nu}{B}}$. We then have
    \begin{align*}
        M_{t+1}(\nu) \leq M_t((a_t + 2b_t^{\prime2}LNB)\nu)\exp(\nu c_tV) \leq \exp\left((a_t + 2b_t^{\prime2}LNB + c_tVB)\mfrac{\nu}{B} \right),
    \end{align*} 
    where we use~\eqref{eq:induction_step} in the first and the induction hypothesis in the second inequality. For our claim to hold, it suffices to show $a_t + 2b_t^{\prime2}LNB + c_tVB \leq 1$. Plugging in the values of $a_t$, $b_t^\prime$ and $c_t$, we have
    \begin{align*}
        a_t + 2b_t^{\prime2}LNB + c_tVB &= \left(1 - \mfrac{2\mu a\gamma}{t+2} \right)\left(\mfrac{t + 2}{t +  1}\right)^{\zeta} + \mfrac{2a^2LNB}{(t + 2)^{2\delta - 2\zeta}(t + 1)^{\zeta}} + \mfrac{a^2VB}{(t + 2)^{2\delta - \zeta}} \\ &\leq \left(\mfrac{t + 2}{t + 1}\right)^{\zeta}\bigg(1 - \mfrac{2\mu a\gamma}{t+2} + \mfrac{2a^2LNB}{(t + 2)^{2\delta - \zeta}} +  \mfrac{a^2VB(t + 1)^\zeta}{(t + 2)^{2\delta}}\bigg) \\ &\leq \left(\mfrac{t + 2}{t + 1}\right)^{\zeta} \left(1 - \mfrac{2\mu a\gamma}{t+2} + \mfrac{2a^2LNB}{(t + 2)^{2\delta - \zeta}} +  \mfrac{a^2VB}{(t + 2)^{2\delta - \zeta}}\right).
    \end{align*} 
    Noticing that $2\delta - \zeta \geq 1$ and setting $B = \min\left\{\frac{1}{(t_0 - 1)^\zeta(f(\bx^{(0)}) - f^\star)}, \frac{\mu\gamma}{aL(2N + \nicefrac{C^2}{2})} \right\}$, gives
    \begin{align*}
        a_t + 2b_t^{\prime 2}LNB + c_tVB \leq \left(\mfrac{t + 2}{t + 1}\right)^{\zeta}\left(1 - \mfrac{\mu a\gamma}{t+2}\right) \leq \exp \lp \mfrac{\zeta}{t + 1} - \mfrac{a\mu\gamma}{t + 2} \rp  \leq 1, 
    \end{align*} 
    where in the second inequality we use $1 + x \leq e^x$, while the third inequality follows from the choice of $\zeta$. Therefore, we have shown that $M_{t}(\nu) \leq e^\frac{\nu}{B}$, for any $t \geq 0$ and any $0 \leq \nu \leq B$. By Markov's inequality, it readily follows that
    \begin{align*}
        \mathbb{P}(f(\bxtp) - f^\star \geq \epsilon) = \mathbb{P}(Y_{t+1} \geq (t + 2)^\zeta\epsilon) \leq e^{-\nu(t + 2)^\zeta\epsilon}M_{t+1}(\nu) \leq e^{1-B(t + 2 )^\zeta\epsilon},
    \end{align*} where in the last inequality we set $\nu = B$. Finally, using strong convexity, we have 
    \begin{align*}
        \mathbb{P}(\|\bxtp - \bx^\star \|^2 \geq \epsilon) \leq \mathbb{P}\left(f(\bxtp) - f^\star \geq \mfrac{\mu}{2}\epsilon\right) \leq ee^{-B(t + 2)^\zeta\mfrac{\mu}{2}\epsilon},
    \end{align*} which implies that, for any $\beta \in (0,1)$, with probability at least $1 - \beta$,
    \begin{equation*}
        \|\bxtp - \bx^\star\|^2 \leq  \mfrac{2\log\left(\nicefrac{e}{\beta}\right)}{\mu B(t+2)^\zeta},
    \end{equation*} completing the proof. 
\end{proof}

\section{Analytical and Numerical Studies}\label{sec:an-num}

In this section we present analytical and numerical studies in the case of strongly convex costs, specializing the rate $\zeta$ from Theorem~\ref{theorem:main} for different nonlinearities. As we proceed to show, the rates provided by our theory, while general, are able to uncover important phenomena, namely which choice of nonlinearity is preferred for the given problem settings. Subsection~\ref{subsec:analytical} provides the analytical study, while Subsection~\ref{subsec:numerical} provides the accompanying numerical results.

\subsection{Analytical Study}\label{subsec:analytical}

In this section we provide an analytical study, using the rate $\zeta$ obtained in Theorem~\ref{theorem:main}, with the goal of showing that clipping, exclusively considered in prior works, is not always the optimal choice of nonlinearity. Recall the constants where $\phi^\prime(0)$, $\xi$ are constants defined in Section~\ref{subsec:theory-nonconv}. We then have the following result, whose proof can be found in Appendix~\ref{app:num}.  

\begin{lemma}\label{lm:analytical}
    For any strongly convex cost with Lipschitz continuous gradients and the noise from Example~\ref{example:1}, a component-wise nonlinearity is preferred to joint clipping whenever
    \begin{equation}\label{eq:duality}
        \frac{\phi^\prime(0)\xi}{C_1} \geq 8\sqrt{d}\left(1 - \frac{1}{(1 + \nicefrac{B_0}{\sqrt{d}})^{\alpha - 1}} \right)^d.
    \end{equation} 
\end{lemma}

For sign and cclip it can be shown that the value $\frac{\phi^\prime(0)\xi}{C_1}$ is approximately $\frac{\alpha - 1}{\alpha}$ and $\frac{m-1}{m}(1 - (m+1)^{-\alpha})$, respectively, where $\alpha > 2$ is the constant from Example~\ref{example:1}, while $m > 1$ is the clipping radius for cclip (see~\cite{jakovetic2023nonlinear}). Clearly, if $B_0$ is fixed, one can find a large enough $d$ for which the relation~\eqref{eq:duality} holds for both sign and cclip. On the other hand, for a fixed $d$, if $B_0 \rightarrow +\infty$, then the relation does not hold. If both $B_0$ and $d$ grow to infinity, we notice two regimes: 
\begin{enumerate}
    \item $B_0 = o(\sqrt{d})$, e.g., $B_0 = d^{\nicefrac{1}{4}}$, the relation~\eqref{eq:duality} holds for sufficiently large $d$.
    \item $d = o(B_0)$, e.g., $B_0 = d^2$, the relation~\eqref{eq:duality} does not hold. 
\end{enumerate} Hence, for $B_0$ finite or growing at an appropriate rate, we can always find a $d$ large enough, for which the relation from Lemma~\ref{lm:analytical} holds. We verify this numerically in Figure~\ref{fig:rhs_behaviour}, where we plot the behaviour of the right-hand side of~\eqref{eq:duality} when $B_0 = d^2$, $B_0 = d^{\nicefrac{1}{4}}$ and $B_0 = 100$ (i.e., constant) on the $y$-axis, versus the problem dimension $d$ on the $x$-axis. The straight lines show the left-hand side of~\eqref{eq:duality}, when specialized to sign and cclip, i.e., $\frac{\alpha - 1}{\alpha}$ and $\frac{m-1}{m}\left(1 - (m+1)^{-\alpha}\right)$, for $\alpha = 2.05$ and $m = 2$. We can clearly see that, as $d$ increases, the right-hand side either blows up ($B_0 = d^2$), rapidly decreases to zero ($B_0 = d^{\nicefrac{1}{4}}$), or initially blows up, but decreases to zero for $d$ sufficiently large ($B_0 = const.$), as claimed. In the subplot we zoom in on the lines representing the behaviour of the left-hand side of~\eqref{eq:duality} for sign and cclip, which suggests that sign is a better choice of nonlinearity in this instance. Therefore, there are regimes for which our theory suggests clipping is \emph{not the optimal choice of nonlinearity}. This is in line with the observations made in~\cite{zhang2020adaptive}, who noted that, compared to joint clipping, cclip converges faster and achieves a better dependence on problem dimension. 

\begin{figure}[ht]
\centering
\includegraphics[scale=0.5]{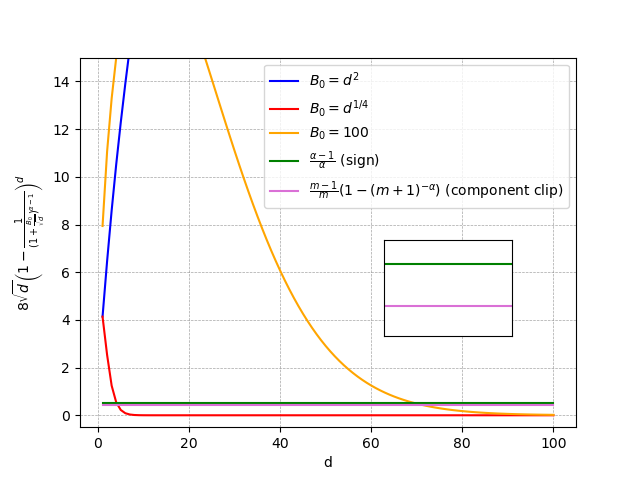}
\caption{Numerical verification of the inequality from Lemma~\ref{lm:analytical}, for different values of $d$ and $B_0$. We can see that the inequality holds for $B_0$ finite or growing at an appropriate rate.}
\label{fig:rhs_behaviour}
\end{figure}

\subsection{Numerical Results}\label{subsec:numerical}

In this section we verify our analytical findings numerically. We consider a quadratic problem $f(\bx) = \frac{1}{2}\bx^\top A\bx + \mathbf{b}^\top\bx$, where $A \in \R^{d \times d}$ is positive definite, with $\mathbf{b} \in \R^d$ a fixed vector. We set $d = 100$. The stochastic noise is generated according to the component-wise noise PDF from Example~\ref{example:1}, with $\alpha = 2.05$. We compare the performance of sign, component-wise and joint clipped SGD, with all three algorithms using the step-size schedule $\alpha_t = \frac{1}{t+2}$. For clipping based algorithms, we choose the clipping thresholds $m$ and $M$ for which component-wise and joint clipped SGD performed the best, that being $m = 1$ and $M = 100$. All three algorithms are initialized at the zero vector and perform $T = 25000$ iterations, across $R = 5000$ runs. To evaluate the performance of the methods, we use the following criteria:
\begin{enumerate}[leftmargin=*]
    \item \emph{Mean-squared error (MSE)}: we present the MSE of the algorithms, by evaluating the model gap $\|\bx^{(t)} - \bx^\star\|^2$ in each iteration, averaged across all runs, i.e., the final estimator at iteration $t = 1,\ldots,T$, is given by $MSE^t = \frac{1}{R}\sum_{r = 1}^R\|\bx^{(t)}_r - \bx^\star\|^2$, where $\bx^{(t)}_r$ is the $t$-th iterate in the $r$-th run, generated by the nonlinear SGD algorithms.

    \item \emph{High-probability estimate}: we evaluate the high-probability behaviour of the methods, as follows. We consider the events $A^t = \{\|\bx^{(t)} - \bx^\star\|^2 > \varepsilon \}$, for a fixed $\varepsilon > 0$. To estimate the probability of $A^t$, for each $t = 1,\ldots, T$, we construct a Monte Carlo estimator of the empirical probability, by sampling $n = 3000$ instances from the $R = 5000$ runs, uniformly with replacement. We then obtain the empirical probability estimator as $\mathbb{P}_n(A^t) = \frac{1}{n}\sum_{i = 1}^n\mathbb{I}_i(A^{t}) = \frac{1}{n}\sum_{i = 1}^n\mathbb{I}\left(\{\|\bx_i^{(t)} - \bx^\star\|^2 > \varepsilon\}\right)$, where $\mathbb{I}(\cdot) \in \{0,1\}$ is the indicator function, with $\bx^{(t)}_i$ the $i$-th Monte Carlo sample. 
\end{enumerate} The results are presented in Figure~\ref{fig:fig1}. We can see that component-wise nonlinearities outperform joint clipping, both in terms of MSE and high-probability performance, thus validating our analytical findings from Section~\ref{sec:an-num}, further underlining the usefulness of the exponent $\zeta$. 

\begin{figure}[ht]
\centering
\begin{tabular}{ccc}
\includegraphics[scale=0.33]{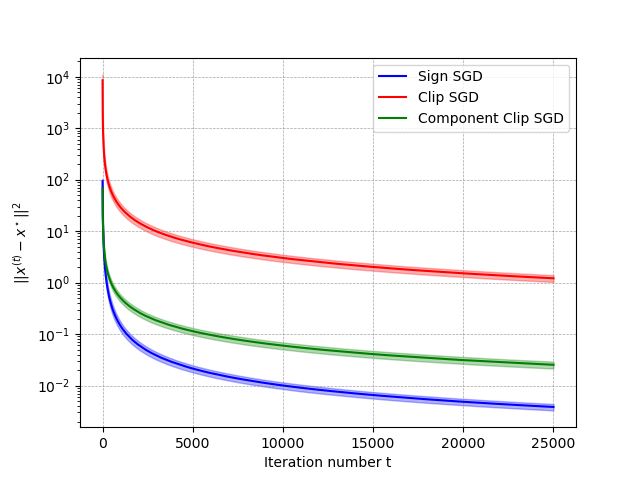}
&
\includegraphics[scale=0.33]{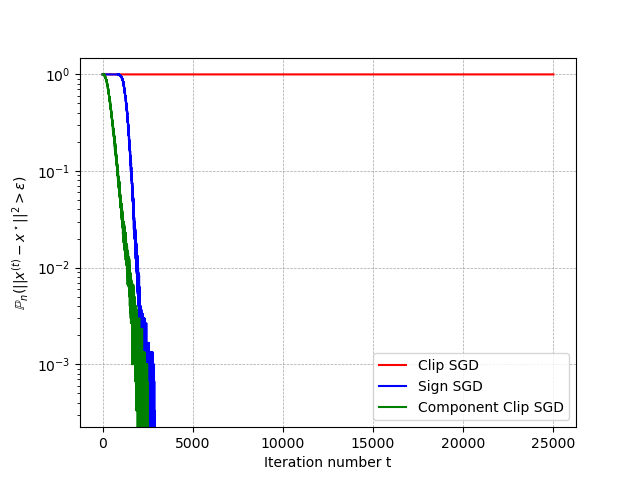}
&
\includegraphics[scale=0.33]{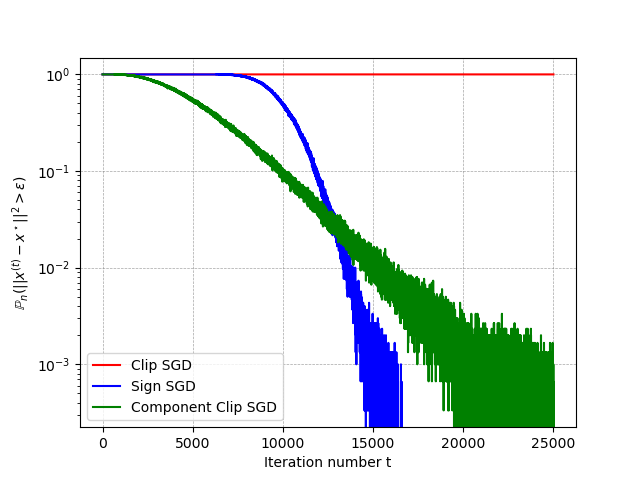}
\end{tabular}
\caption{Performance of sign, cclip and joint clipping for $d = 10$. Left to right: MSE performance and high-probability performance for $\varepsilon = \{0.1, 0.01 \}$, respectively.}
\label{fig:fig1}
\end{figure}

\section{Conclusion}\label{sec:conclusion}
We present high-probability convergence guarantees for a broad class of nonlinear streaming SGD algorithms, under heavy-tailed noise. Our results are built on a general framework, that encompasses many popular nonlinear versions of SGD, such as clipped, normalized, quantized and sign SGD, providing high-probability convergence guarantees for both non-convex and strongly convex functions. Compared to state-of-the-art works~\cite{nguyen2023improved,sadiev2023highprobability}, we extend the high-probability convergence guarantees to novel nonlinearities, relax the noise moment condition, and demonstrate regimes in which our convergence rates are better than state-of-the-art. Moreover, for strongly convex functions we show that our rates are informative for the optimal choice of nonlinearity and that clipping, exclusively considered in prior works, is not always the optimal choice of nonlinearity, further highlighting the importance of our general framework. Numerical results confirm the theoretical findings and show the usefulness of our general framework for informing on the optimal choice of nonlinearity.

\bibliography{bibliography}

\appendix

\section{Introduction}

The Appendix presents results omitted from the main body. Appendix~\ref{app:facts} provides some useful facts and results used in the proofs. Appendix~\ref{app:proofs} provides the proofs omitted from Section~\ref{sec:main}. Appendix~\ref{app:rate} provides rate expressions for component-wise and joint nonlinearities. Appendix~\ref{app:num} presents results omitted from Section~\ref{sec:an-num}. Appendix~\ref{app:further} provides some further derivations.

\section{Useful Facts}\label{app:facts}
In this section we present some useful facts and results, concerning $L$-smooth, $\mu$-strongly functions, bounded random vectors and the behaviour of nonlinearities.

\begin{fact}
    Let $f: \R^d \mapsto \R$ be $L$-smooth, $\mu$-strongly convex, and let $\bx^\star = \argmin_{\bx \in \R^d}f(\bx)$. Then, for any $\bx \in \R^d$, we have
    \begin{equation*}
        2\mu\left(f(\bx) - f(\bx^\star)\right)\leq \|\nabla f(\bx)\|^2 \leq 2L\left(f(\bx) - f(\bx^\star)\right).
    \end{equation*}
\end{fact}
\begin{proof}
    The proof of the upper bound follows by plugging $y = \bx$, $x = \bx^\star$ in equation (2.1.10) of Theorem~2.1.5 from~\cite{nesterov-lectures_on_cvxopt}. The proof of the lower bound similarly follows by plugging $y = \bx$, $x = \bx^\star$ in equation (2.1.24) of Theorem~2.1.10 from~\cite{nesterov-lectures_on_cvxopt}.
\end{proof}

\begin{fact}\label{fact:subgauss}
    Let $X \in \R^d$ be a zero-mean, bounded random vector, i.e., $\E X = 0$ and $\|X\| \leq \sigma$, for some $\sigma > 0$. Then, $X$ is sub-Gaussian, i.e., there exists a positive constant $N = N(\sigma)$ such that, for any $\lambda \in \R^d$, we have
    \begin{equation*}
        \E e^{\langle X,\lambda \rangle} \leq e^{\frac{N\|\lambda\|^2}{2}}.
    \end{equation*}
\end{fact}
\begin{proof}
    We begin by first showing that a scalar Rademacher random variable is sub-Gaussian. Recall that a random variable $\varepsilon$ is a Rademacher random variable, if $\varepsilon$ takes the values $-1$ and $1$, both with probability $1/2$. We then have, for any $t \in \R$
    \begin{align}\label{eq:rademacher}
        \E e^{\varepsilon t} \stackrel{(a)}{=} \frac{1}{2}\left(e^{t} + e^{-t}\right) \stackrel{(b)}{=} \sum_{k = 0}^{\infty}\frac{t^{2k}}{{(2k)!}} \stackrel{(c)}{\leq} \sum_{k = 0}^{\infty}\frac{t^{2k}}{2^kk!} = \sum_{k = 0}^{\infty}\frac{(\nicefrac{t^2}{2})^k}{k!} = e^{\frac{t^2}{2}},
    \end{align} where $(a)$ follows from the definition of the Rademacher random variable, $(b)$ follows from the fact that $e^t = \sum_{k = 0}^\infty \frac{t^k}{k!}$, for any $t \in \R$, while $(c)$ follows from the fact that $2k! \geq 2^kk!$, for any $k \in \N_0$. Let $X^\prime \in \R^d$ be an independent, identically distributed copy of $X$. We then have, for any $\lambda \in \R^d$
    \begin{align*}
        \E e^{\langle X,\lambda \rangle} \stackrel{(a)}{=} \E e^{\langle X - \E X^\prime,\lambda \rangle} \stackrel{(b)}{\leq} \E_{X,X^\prime} e^{\langle X - X^\prime,\lambda \rangle} &\stackrel{(c)}{=} \E_{X,X^\prime}\E_\varepsilon e^{\varepsilon\langle X - X^\prime,\lambda \rangle} \stackrel{(d)}{\leq} \E_{X,X^\prime} e^{\frac{\left(\langle X - X^\prime,\lambda \rangle\right)^2}{2}} \\ &\leq \E_{X,X^\prime} e^{\frac{\| X - X^\prime\|^2\|\lambda\|^2}{2}} \stackrel{(e)}{\leq} e^{2\sigma^2\|\lambda\|^2},
    \end{align*} where $(a)$ follows from the fact that $X$ is zero-mean, in $(b)$ we use Jensen's inequality, $(c)$ uses the fact that $X,X^\prime$ are i.i.d., therefore $X - X^\prime$ has the same distribution as $\varepsilon(X - X^\prime)$, where $\varepsilon$ is a Rademacher random variable, in $(d)$ we use~\eqref{eq:rademacher}, while $(e)$ follows from the boundedness of $X$. Choosing $N = 4\sigma^2$, the desired inequality follows.  
\end{proof}

\section{Missing Proofs}\label{app:proofs}

In this section we provide proofs of Lemmas \ref{lm:error_component},~\ref{lm:huber} and \ref{lm:key}, omitted from the main body. We begin by proving Lemma~\ref{lm:error_component}.

\begin{proof}[Proof of Lemma~\ref{lm:error_component}]
    Recall the definition of the error vector $\bet \triangleq \boldsymbol{\Phi}^{(t)} - \boldsymbol{\Psi}^{(t)}$, where $\bPhi^{(t)} \triangleq \mbe_\bzt \left[\bPsi (\nabla f(\bxt)+\bzt)\right]$ is the denoised version of $\bPsi^{(t)}$. By definition, it then follows that
    \begin{equation*}
        \E\left[\bet \vert \: \mathcal{F}_t\right] = \E\left[\bPhi^{(t)} - \bPsi^{(t)} \vert \: \mathcal{F}_t\right] = \bPhi^{(t)} - \E\left[ \bPsi^{(t)} \vert \: \mathcal{F}_t \right] = 0,
    \end{equation*} where the last equality follows from the fact that $\bPhi^{(t)}$ is $\mathcal{F}_t$-measureable and $\E\left[ \bPsi^{(t)} \vert \: \mathcal{F}_t \right] = \E_\bzt\left[ \bPsi(\nabla f(\bxt)+ \bzt) \right] = \bPhit$. Moreover, by Assumption~\ref{asmpt:nonlin}, we have
    \begin{equation*}
        \|\bet \| = \|\bPhi^{(t)} - \bPsi^{(t)}\| \leq \|\bPhi^{(t)}\| + \|\bPsi^{(t)}\| \leq \E\|\bPsi^{(t)}\| +  C \leq 2C,
    \end{equation*} which proves the first claim. The second claim readily follows by using the fact that $\bet$ is a bounded random variable and applying Fact~\ref{fact:subgauss}.
\end{proof}

We next prove Lemma~\ref{lm:huber}.

\begin{proof}[Proof of Lemma~\ref{lm:huber}]
    Using the results from Lemma~\ref{lm:polyak-tsypkin}, for any $x \in \R$, and any $i \in [d]$, we have
    \begin{equation*}
        \phi_i(x) = \phi_i(0) + \phi_i^\prime(0)x + h_i(x)x = \phi_i^\prime(0)x + h_i(x)x,
    \end{equation*} where $h_i: \R \mapsto \R$ is such that $\lim_{x \rightarrow 0}h_i(x) = 0$. Recalling that $\phi^\prime(0) \triangleq \min_{i \in [d]}\phi_i^\prime(0) > 0$, it follows that there exists a $\xi > 0$ (depending only on $\bPhi$) such that, for each $x \in \R$ and all $i \in [d]$, we have $|h_i(x)| \leq \phi^\prime(0) / 2$, if $|x| \leq \xi$. Therefore, for any $0 \leq x \leq \xi$, we have
    \begin{equation}\label{eq:1}
        \phi_i(x) \geq \frac{\phi^\prime(0)x}{2}.
    \end{equation} On the other hand, for $x > \xi$, since $\phi_i$ is non-decreasing, we have from \eqref{eq:1} that $\phi_i(x) \geq \phi_i(\xi) \geq \frac{\phi^\prime(0)\xi}{2}$. Therefore, it follows that, for any $x \geq 0$
    \begin{equation}\label{eq:2}
        \phi_i(x) \geq \frac{\phi^\prime(0)}{2}\min\{x,\xi \}.
    \end{equation} Next, for any $a \in \R$, using oddity of $\phi_i$, we get
    \begin{equation}\label{eq:3}
        a\phi_i(a) = |a|\phi_i(|a|).
    \end{equation} Using~\eqref{eq:2} and~\eqref{eq:3}, we then have, for any vector $\bx \in \R^d$
    \begin{align}\label{eq:4}
        \langle \bx, \bPhi(\bx) \rangle &= \sum_{i = 1}^d x_i\phi_i(x_i)  \stackrel{\eqref{eq:3}}{=} \sum_{i = 1}^d |x_i|\phi(|x_i|) \geq \max_{i \in [d]}|x_i|\phi_i(|x_i|) \stackrel{\eqref{eq:2}}{\geq} \frac{\phi^\prime(0)}{2}\max_{i \in [d]}\min\{\xi|x_i|,|x_i|^2\} \nn \\ &= \frac{\phi^\prime(0)}{2}\min\{\xi\|\bx\|_{\infty},\|\bx\|_{\infty}^2\} \geq \frac{\phi^\prime(0)}{2}\min\{\nicefrac{\xi\|\bx\|}{ \sqrt{d}},\nicefrac{\|\bx\|^2}{d}\},
    \end{align} where the last inequality follows from the fact that $\|\bx\|_{\infty} \geq \|\bx\| / \sqrt{d}$.
\end{proof}

In order to prove Lemma~\ref{lm:key}, we first state and prove some intermediate results.

\begin{lemma}\label{lm:gradient-bound}
    Let Assumptions~\ref{asmpt:nonlin}-\ref{asmpt:noise} hold, with the step-size given by $\alpha_t = \frac{a}{(t+t_0)^\delta}$, for any $\delta \in (0.5,1)$ and $t_0 > 1$. Then, for any $t \in \N_0$, we have
    \begin{equation*}
        \|\nabla f(\bxt)\| \leq G_t \triangleq L\left(\|\bx^{(0)} - \bx^\star\| + aC\right)\frac{(t+t_0)^{1-\delta}}{1 - \delta}.
    \end{equation*}
\end{lemma}
\begin{proof} Using $L$-smoothness of $f$ and the update~\eqref{eq:nonlin-sgd}, we have
    \begin{align}
        \|\nabla f(\bxt)\| &\leq L\|\bxt - \bx^\star\| = L\|\bx^{(t-1)} - \alpha_{t-1}\bPsi^{(t-1)} - \bx^\star\| \nn \\ 
        &\leq L\left(\|\bx^{(t-1)} - \bx^\star\| + \alpha_{t-1}\|\boldsymbol{\Psi}^{(t-1)}\| \right) \nn \\ 
        &\leq L\left(\|\bx^{(t-1)} - \bx^\star\| + \alpha_{t-1}C \right). \label{eq_proof:thm:nonL_cw_1}
    \end{align}
    Unrolling the recursion in \eqref{eq_proof:thm:nonL_cw_1}, we get
    \begin{align*}
        \|\nabla f(\bxt)\| \leq L \|\bx^{(0)} - \bx^\star\| + LC\sum_{k = 1}^t\alpha_{k-1} \leq L\left(\|\bx^{(0)} - \bx^\star\| + aC\right)\frac{(t+t_0)^{1-\delta}}{1 - \delta}, 
    \end{align*} completing the proof.
\end{proof}

The next result characterize the behaviour of the nonlinearity, when it takes the form $\bPsi(\bx) = \lbr \calN_1(x_1), \dots, \calN_1(x_d) \rbr^\top$. It follows a similar idea to Lemma~5.5 from~\cite{jakovetic2023nonlinear}, with the main difference due to allowing for potentially different marginal PDFs of each noise component. Since the proof follows the same steps, we omit it for brevity.

\begin{lemma}
\label{lemma:1.1}
    Let Assumptions~\ref{asmpt:nonlin}-\ref{asmpt:noise} hold and the nonlinearity $\bPsi$ be component-wise, i.e., of the form $\bPsi(\bx) = \lbr \calN_1(x_1), \dots, \calN_1(x_d) \rbr^\top$. Then, there exists a positive constant $\xi$ such that, for any $t \in \N_0$, there holds almost surely for each $j = 1,\ldots,d$, that $|\phi^{(t)}_i| \geq |[\nabla f(\bxt)]_i|\frac{\phi_i^\prime(0)\xi}{2G_t}$, where $G_t$ is defined in Lemma~\ref{lm:gradient-bound}, while $\phi_i^\prime(0) = \frac{\partial}{\partial x_i}\E_{z_i}\calN_1(x_i + z_i)\:\big\vert_{x_i = 0}$.
\end{lemma}

The next result is a restatement of Lemma~6.2 in~\cite{jakovetic2023nonlinear}, in the case when the nonlinearity takes the form $\bPsi(\bx) = \bx\calN_2(\|\bx\|)$. We omit the proof for brevity.

\begin{lemma}[Lemma 6.2 from \cite{jakovetic2023nonlinear}]\label{lemma:1.2}
    Let Assumptions~\ref{asmpt:nonlin}-\ref{asmpt:noise} hold and the nonlinearity be of the form $\bPsi(\bx) = \bx\calN_2(\|\bx\|)$. Then, the following holds
    \begin{equation*}
        \left\langle\bPhi(\bx),\bx\right\rangle \geq 2(1-\kappa)\|\bx\|^2\int_{\mathcal{J}(\bx)}\calN_2(\|\bx\|+\|\bz\|)p(\bz)d\bz,
    \end{equation*} where $\J(\bx) = \left\{\bz \in \R^d: \frac{\langle\bz, \bx\rangle}{\|\bz\|\|\bx\|} \in [0,\kappa]\right\}$ and $\kappa \in (0,1)$ is a constant. 
\end{lemma}

The next result characterizes the behaviour of the nonlinearity, when it takes the form $\bPsi(\bx) = \bx\calN_2(\|\bx\|)$. It builds on Lemma~\ref{lemma:1.2} and we provide the full proof for completeness.

\begin{lemma}\label{lm:10}
    Let Assumptions~\ref{asmpt:nonlin}-\ref{asmpt:noise} hold and the nonlinearity be of the form $\bPsi(\bx) = \bx\calN_2(\|\bx\|)$. Then, there exists a constant $\kappa \in (0,1)$ such that for any $t \in \N_0$, there holds almost surely that  
    \begin{equation*}
        \langle\nabla f(\bxt), \boldsymbol{\Phi}^{(t)}\rangle \geq \mfrac{2(1 - \kappa)\lambda(\kappa)\calN_2(1)\|\nabla f(\bxt)\|^2}{B_0 + G_t}, 
    \end{equation*} where $\kappa \in (0,1)$ is the constant from Lemma~\ref{lemma:1.2}, $B_0 > 0$ is defined in Assumption~\ref{asmpt:noise}, $G_t$ is defined in Lemma~\ref{lm:gradient-bound}, with $\lambda(\kappa) > 0$ a constant such that $\int_{\left\{\bz \in \mathbb{R}^d: \frac{\langle \bz, \bx\rangle}{\|\bz\|\|\bx\|} \in [0,\kappa], \: \|\bz\| \leq B_0 \right\}}p(\bz)d\bz > \lambda(\kappa),$ for any $\bx$.
\end{lemma}
\begin{proof}
    We start from Lemma \ref{lemma:1.2}, which tells us that, for some $\kappa \in (0,1)$, almost surely 
\begin{equation}\label{eq:ineq1}
        \langle\bPhi^{(t)},\nabla f(\bxt)\rangle \geq 2(1-\kappa)\|\nabla f(\bxt)\|^2\int_{\mathcal{J}(\nabla f(\bxt))}\calN_2(\|\nabla f(\bxt)\|+\|\bz\|)p(\bz)d\bz,
    \end{equation} where $\mathcal{J}(\nabla f(\bxt)) = \left\{\bz \in \R^d: \frac{\langle\bz, \nabla f(\bxt)\rangle}{\|\bz\|\|\nabla f(\bxt)\|} \in [0,\kappa]\right\}$. Note that, as $x\calN_2(x)$ is a non-decreasing function (Assumption~\ref{asmpt:nonlin}), $\calN_2$ satisfies 
    \begin{equation}\label{eq:calN1}
        \calN_2(x) \geq \min\left\{\frac{\calN_2(1)}{x},\calN_2(1) \right\},
    \end{equation} for any $x > 0$. In particular, for any $\bz$ such that $\|\bz\| \leq B_0$, we have $\calN_2(\|\nabla f(\bxt)\| + \|\bz\|) \geq \min\left\{\frac{\calN_2(1)}{\|\nabla f(\bxt)\|+ B_0},\calN_2(1) \right\}$. We then have
    \begin{align}
        \|\nabla f(\bxt)\|^2\int_{\mathcal{J}(\nabla f(\bxt))}&\calN_2(\|\nabla f(\bxt)\|+\|\bz\|)p(\bz)d\bz \geq \|\nabla f(\bxt)\|^2\int_{\mathcal{J}_1}\calN_2(\|\nabla f(\bxt)\|+\|\bz\|)p(\bz)d\bz \nn \\ &\stackrel{(a)}{\geq} \|\nabla f(\bxt)\|^2\int_{\mathcal{J}_\kappa}\min\left\{\frac{\calN_2(1)}{\|\nabla f(\bxt)\|+ B_0},\calN_2(1) \right\}p(\bz)d\bz \nn \\ &\stackrel{(b)}{\geq} \frac{\|\nabla f(\bxt)\|^2\calN_2(1)}{G_t+ B_0}\int_{\mathcal{J}_\kappa}p(\bz)d\bz \nn \\ &\stackrel{(c)}{\geq} \frac{\|\nabla f(\bxt)\|^2\lambda(\kappa)\calN_2(1)}{G_t+ B_0} \label{eq:ineq2}, 
    \end{align} where $\J_\kappa =\left\{\bz \in \R^d: \frac{\langle\bz, \nabla f(\bxt)\rangle}{\|\bz\|\|\nabla f(\bxt)\|} \in [0,\kappa], \: \|\bz\|\leq B_0 \right\}$, $(a)$ follows from~\eqref{eq:calN1}, $(b)$ follows from Lemma~\ref{lm:gradient-bound} and the fact that $G_t > 1$, for every $t \geq 0$, while $(c)$ follows from Assumption~\ref{asmpt:noise}. Combining~\eqref{eq:ineq1} and~\eqref{eq:ineq2}, we have that, almost surely
    \begin{equation*}
        \langle\bPhi^{(t)},\nabla f(\bxt)\rangle \geq \frac{2(1-\kappa)\lambda(\kappa)\calN_2(1)\|\nabla f(\bxt)\|^2}{G_t+ B_0},
    \end{equation*} which is what we wanted to show.
\end{proof}

We are now ready to prove Lemma~\ref{lm:key}.

\begin{proof}[Proof of Lemma~\ref{lm:key}]
    First, consider the case when the nonlinearity is of the form $\bPsi(\bx) = \lbr \calN_1(x_1), \dots, \calN_1(x_d) \rbr^\top$. We then have 
    \begin{align*}
        \langle \bPhit, \nabla f(\bxt)\rangle &= \sum_{i = 1}^d \phi^{(t)}_i [\nabla f(\bxt)]_i \stackrel{(a)}{=} \sum_{i = 1}^d |\phi^{(t)}_i| |[\nabla f(\bxt)]_i| \\ & \stackrel{(b)}{\geq} \sum_{i = 1}^d |[\nabla f(\bxt)]_i|^2\frac{\phi_i^\prime(0)\xi}{2G_t} \stackrel{(c)}{\geq} \frac{\phi^\prime(0)\xi}{2G_t}\|\nabla f(\bxt)\|^2 = \gamma(t + t_0)^{\delta - 1}\|\nabla f(\bxt)\|^2,
    \end{align*} where $\gamma = \frac{(1 - \delta)\phi^\prime(0)\xi}{2L\left(\|\bx^{(0)} - \bx^\star\| + aC\right)}$, $(a)$ follows from the oddity of $\calN_1$, $(b)$ follows from Lemma~\ref{lemma:1.1}, $(c)$ follows from $\phi^\prime(0) = \min_{i =1,\ldots,d}\phi_i^\prime(0)$. On the other hand, if the nonlinearity is of the form $\bPsi(\bx) = \bx\calN_2(\|\bx\|)$, we get
    \begin{align*}
        \langle \bPhit, \nabla f(\bxt)\rangle \geq \frac{2(1-\kappa)\lambda(\kappa)\calN_2(1)\|\nabla f(\bxt)\|^2}{G_t+ B_0} \geq \gamma(t + t_0)^{\delta - 1}\|\nabla f(\bxt)\|^2,
    \end{align*} where $\gamma = \frac{2(1-\delta)(1-\kappa)\lambda(\kappa)\calN_2(1)}{L\left(\|\bx^{(0)} - \bx^\star\| + aC\right) + B_0}$, the first inequality follows from Lemma~\ref{lm:10}, while the second follows from the definition of $G_t$ and the fact that $G_t + B_0 \leq (L\left(\|\bx^{(0)} - \bx^\star\| + aC\right) + B_0)\frac{(t+t_0)^{1 - \delta}}{1 - \delta}$. This completes the proof.
\end{proof}

\section{Rate $\zeta$}\label{app:rate}

Recalling Assumption~\ref{asmpt:nonlin} and the definition of $C$, it readily follows that $\gamma = \frac{(1 - \delta)\phi^\prime(0)\xi}{2L\left(\|\bx^{(0)} - \bx^\star\| + a\sqrt{d}C_1\right)}$ for nonlinearities of the form $\bPsi(\bx) = \lbr \calN_1(x_1), \dots, \calN_1(x_d) \rbr^\top$ (i.e., component-wise), while $\gamma = \frac{2(1-\delta)(1-\kappa)\lambda(\kappa)\calN_2(1)}{L\left(\|\bx^{(0)} - \bx^\star\| + aC_2\right) + B_0}$, for nonlinearities of the form $\bPsi(\bx) = \bx\calN_2(\|\bx\|)$ (i.e., joint). Combined with Theorem~\ref{theorem:main}, it follows that the rate $\zeta$ is given by
\begin{align*}
    \zeta_{joint} &= \min\left\{2\delta - 1, \frac{2a\mu(1 - \kappa)\lambda(\kappa)(1-\delta)\calN_2(1)}{L\left(\|\bx^{(0)} - \bx^\star\| + aC_2 \right) + B_0} \right\}, \\ \zeta_{comp} &= \min\left\{2\delta - 1, \frac{a\mu\phi^\prime(0)\xi(1-\delta)}{4L\left(\|\bx^{(0)} - \bx^\star\| + aC_1\sqrt{d} \right)} \right\}.
\end{align*} We note that $\zeta$ depends on the following problem-specific parameters:
\begin{itemize}[leftmargin=*]
    \item \emph{Initialization} - starting 
    farther from the true minima $\bx^\star$,
    results in smaller $\zeta$. The effect of initialization can be eliminated by choosing sufficiently large $a$.
    \item \emph{Condition number} - larger values of $\frac{L}{\mu}$ (i.e., a more difficult problem) result in smaller $\zeta$.
    \item \emph{Nonlinearity} - the dependence of $\zeta$ on the nonlinearity comes in the form of two terms: the uniform bound on the nonlinearity $C_1$ ($C_2$), and the value $\phi^\prime(0)$ ($\calN_2(1)$).
    \item \emph{Problem dimension} - for component-wise nonlinearities directly in the form of $\sqrt{d}$; for joint ones indirectly, in the form of $\calN_2(1)$. As we show in Section~\ref{sec:an-num}, the dependence on dimension for joint nonlinearities, while not explicit, can sometimes be worse than that of component-wise ones.
    \item \emph{Noise} - in the form of $\phi^\prime(0)$ and $\xi$ for component-wise and $B_0$, $\lambda(\kappa)$ for joint ones.  
    \item \emph{Step-size} - both terms in the definition of $\zeta$ depend on the step-size parameter $\delta \in (0,1)$.
\end{itemize}

\section{Miscellaneous}\label{app:num}

In this section we prove Lemma~\ref{lm:analytical}.

\begin{proof}[Proof of Lemma~\ref{lm:analytical}]
    Consider clipping and a generic component-wise nonlinearity. From Appendix~\ref{app:rate} we know that the rate $\zeta$ is then given by
    \begin{align*}
        \zeta_{clip} &= \min\left\{2\delta - 1, \frac{2a\mu(1 - \kappa)\lambda(\kappa)(1-\delta)\calN_2(1)}{L\left(\|\bx^{(0)} - \bx^\star\| + aC_2 \right) + B_0} \right\}, \\ \zeta_{comp} &= \min\left\{2\delta - 1, \frac{a\mu\phi^\prime(0)\xi(1-\delta)}{4L\left(\|\bx^{(0)} - \bx^\star\| + aC_1\sqrt{d} \right)} \right\},
    \end{align*} where $C_1,\: C_2 > 0$ are the bounds on the nonlinearities, $\xi > 0$ is a constant depending on $\phi$, $\kappa \in (0,1)$ and $\lambda(\kappa) > 0$ being a constant that satisfies
    \begin{equation*}
        \int_{\left\{\bz \in \mathbb{R}^d: \frac{\langle \bz, \bx\rangle}{\|\bz\|\|\bx\|} \in [0,\kappa], \: \|\bz\| \leq B_0 \right\}}p(\bz)d\bz > \lambda(\kappa).
    \end{equation*} Our goal is to show that clipping is not the best choice of nonlinearity. We will do so by showing that $\zeta_{clip} \leq \zeta_{comp}$. To guarantee this is the case, it suffices that
    \begin{equation*}
        \frac{2a\mu(1 - \kappa)\lambda(\kappa)(1-\delta)\calN_2(1)}{L\left(\|\bx^{(0)} - \bx^\star\| + aC_2 \right) + B_0} \leq \frac{a\mu\phi^\prime(0)\xi(1-\delta)}{4L\left(\|\bx^{(0)} - \bx^\star\| + aC_1\sqrt{d} \right)},
    \end{equation*} or equivalently,
    \begin{equation*}
        \frac{8(1 - \kappa)\lambda(\kappa)\calN_2(1)}{\|\bx^{(0)} - \bx^\star\| + aM + \nicefrac{B_0}{L}} \leq \frac{\phi^\prime(0)\xi}{\|\bx^{(0)} - \bx^\star\| + aC_1\sqrt{d}},
    \end{equation*} where we used the fact that $C_2 = M$, where $M > 0$, is the clipping threshold. Rearranging, we get
    \begin{equation*}
         \phi^\prime(0)\xi \geq 8(1 - \kappa)\lambda(\kappa)\calN_2(1)\frac{\|\bx^{(0)} - \bx^\star\| + aC_1\sqrt{d}}{\|\bx^{(0)} - \bx^\star\| + aM + \nicefrac{B_0}{L}}.
    \end{equation*} Choosing $a > 0$ sufficiently large, it suffices to verify that the following holds
    \begin{equation}\label{eq:verify}
        \frac{\phi^\prime(0)\xi}{C_1} \geq (1 - \kappa)\lambda(\kappa)\calN_2(1)\frac{8\sqrt{d}}{M}.
    \end{equation} Next, recall that the noise PDF from Example~\ref{example:1} is given by
    \begin{equation*}
        p(x) = \frac{\alpha - 1}{2(1 + |x|)^{\alpha}},
    \end{equation*} for some $\alpha > 2$. From the independence of noise components, we know that the joint PDF is given by
    \begin{equation*}
        p(\bz) = p(z_1)\times p(z_2)\times \ldots \times p(z_d),
    \end{equation*} which clearly satisfies Assumption~\ref{asmpt:noise} for any $B_0 > 0$. Next, for any $\bx = \begin{bmatrix} x_1 & \ldots & x_d\end{bmatrix}^\top$ and $\kappa \in (0,1)$, define the set $\mathcal{J}(\kappa)$ as
    \begin{equation*}
        \mathcal{J}(\kappa) = \left\{\bz \in \mathbb{R}^d: \frac{\langle \bz, \bx\rangle}{\|\bz\|\|\bx\|} \in [0,\kappa], \: \|\bz\| \leq B_0 \right\}.
    \end{equation*} Consider the following sets
    \begin{equation*}
        \mathcal{J}_i(\kappa) = \left\{z \in \R: \: \frac{z\cdot x_i}{\|\bx\|} \in [0,\nicefrac{\kappa}{d}], \: |z| \leq \frac{B_0}{\sqrt{d}}\right\}, \: \text{for all } i = 1,\ldots,d,
    \end{equation*} and define the set $\mathbf{B}_1 = \left\{\bz \in \R^d: \|z\| = 1\right\}$. Then, clearly, for the set
    \begin{equation*}
        \widetilde{\mathcal{J}}(\kappa) = \left\{\bz \in \mathbf{B}_1: \: z_i \in \mathcal{J}_i(\kappa), \text{ for all } i = 1,\ldots,d \right\},
    \end{equation*} we have $\widetilde{\mathcal{J}}(\kappa) \subseteq \mathcal{J}(\kappa)$, hence
    \begin{equation*}
        \int_{\mathcal{J}(\kappa)}p(\bz)d\bz > \int_{\widetilde{\mathcal{J}}(\kappa)}p(\bz)d\bz \triangleq \lambda(\kappa). 
    \end{equation*} To prove our claim, it suffices to show that the relation~\eqref{eq:verify} holds when $\lambda(\kappa)$ is replaced by some upper bound. To that end, by definition, we have $\widetilde{\mathcal{J}}(\kappa) \subseteq \left\{\bz \in \R^d: \: z_i \in \mathcal{J}_i(\kappa), \: i = 1,\ldots,d \right\}$, which implies
    \begin{equation*}
        \lambda(\kappa) \leq \prod_{i = 1}^d\int_{\mathcal{J}_i(\kappa)}p_i(z)dz \stackrel{(a)}{=} \left(\int_{\mathcal{J}_1(\kappa)}p(z_1)dz \right)^d \stackrel{(b)}{\leq} \left(\int_{|z| \leq \frac{B_0}{\sqrt{d}}}p(z)dz \right)^d \triangleq \widetilde{\lambda}, 
    \end{equation*} where $(a)$ follows from the fact that all components of $\bz$ are i.i.d., while $(b)$ follows from the fact that $\mathcal{J}_1(\kappa) \subseteq \left\{z \in \R: \: |z| \leq \frac{B_0}{\sqrt{d}} \right\}$. We can compute $\widetilde{\lambda}$ explicitly, by solving the integral, to get
    \begin{equation*}
        \widetilde{\lambda} = \left(2\int_0^{\nicefrac{B_0}{\sqrt{d}}}\frac{\alpha - 1}{2(1 + z)^\alpha}dz\right)^d = \left(1 - \frac{1}{\left(1 + \nicefrac{B_0}{\sqrt{d}}\right)^{\alpha - 1}}\right)^d.
    \end{equation*} Plugging $\widetilde{\lambda}$ in~\eqref{eq:verify}, we get
    \begin{equation*}
        \frac{\phi^\prime(0)\xi}{C_1} \geq 8(1 - \kappa)\left(1 - \frac{1}{(1 + \nicefrac{B_0}{\sqrt{d}})^{\alpha - 1}} \right)^d\calN_2(1)\frac{\sqrt{d}}{M}. 
    \end{equation*} Notice that $\calN_2(1) = \min\{1,M\}$, which implies $\nicefrac{\calN_2(1)}{M} \leq 1$. Moreover, since the right-hand side (RHS) of the above equation is maximized for $\kappa \rightarrow 0$, we get the following relation
    \begin{equation*}
        \frac{\phi^\prime(0)\xi}{C_1} \geq 8\sqrt{d}\left(1 - \frac{1}{(1 + \nicefrac{B_0}{\sqrt{d}})^{\alpha - 1}} \right)^d, 
    \end{equation*} completing the proof.
\end{proof}

\section{Further derivations}\label{app:further}

In this Appendix we show how our bounds from Theorems~\ref{thm:non-conv} and~\ref{thm:polyak} lead to bounds on the best iterate and weighted average of iterates, respectively. We begin by showing the former. In particular, from Theorem~\ref{thm:non-conv} we have
\begin{equation}\label{eq:bound}
    \sum_{k = 0}^{t-1}\widetilde{\alpha}_k\min\{\nicefrac{\xi\|\nabla f(\bxk)\|}{ \sqrt{d}},\nicefrac{\|\nabla f(\bxk)\|^2}{d}\} \leq \frac{R(a,\beta,\delta)}{(t+2)^{1-\delta} - 2^{1-\delta}}.
\end{equation} Define $U \triangleq \{k \in \{0,\ldots,t-1\}: \: \|\nabla f(\bxk)\| \leq \xi\sqrt{d} \}$, with $U^c \triangleq \{0,1,\ldots,t-1\} \setminus U$. From~\eqref{eq:bound}, we then have 
\begin{align*}
    &\sum_{k \in U}\widetilde{\alpha}_k\|\nabla f(\bxk)\|^2 \leq \frac{dR(a,\beta,\delta)}{(t+2)^{1-\delta} - 2^{1-\delta}}, \\
    &\sum_{k \in U^c}\widetilde{\alpha}_k\|\nabla f(\bxk)\| \leq \frac{R(a,\beta,\delta)\sqrt{d}/\xi}{(t+2)^{1-\delta} - 2^{1-\delta}}.
\end{align*} It then readily follows that
\begin{align*}
    \min_{0 \leq k \leq t-1}\|\nabla f(\bxk)\| \leq \sum_{k \in U}\widetilde{\alpha}_k\|\nabla f(\bxk)\| + \sum_{k \in U^c}\widetilde{\alpha}_k\|\nabla f(\bxk)\| \leq \sum_{k = 0}^{t-1}\widetilde{\alpha}_kz_k + \frac{R(a,\beta,\delta)\sqrt{d}/\xi}{(t+2)^{1-\delta} - 2^{1-\delta}},
\end{align*} where $z_k = \|\nabla f(\bxk)\|$, for $k \in U$, otherwise $z_k = 0$. Using Jensen's inequality, we get
\begin{align*}
    \min_{0 \leq k \leq t-1}\|\nabla f(\bxk)\| &\leq \sqrt{\sum_{k 
= 0}^{t-1}\widetilde{\alpha}_kz_k^2} + \frac{R(a,\beta,\delta)\sqrt{d}/\xi}{(t+2)^{1-\delta} - 2^{1-\delta}} = \sqrt{\sum_{k \in U}\widetilde{\alpha}_k\|\nabla f(\bxk)\|^2} + \frac{R(a,\beta,\delta)\sqrt{d}/\xi}{(t+2)^{1-\delta} - 2^{1-\delta}} \\ &\leq \sqrt{\frac{dR(a,\beta,\delta)}{(t+2)^{1-\delta} - 2^{1-\delta}}} + \frac{R(a,\beta,\delta)\sqrt{d}/\xi}{(t+2)^{1-\delta} - 2^{1-\delta}},
\end{align*} giving the bound $\min_{0 \leq k \leq t-1}\|\nabla f(\bxk)\| = \mathcal{O}\left(t^{\nicefrac{(\delta - 1)}{2}} \right)$. Next, from Theorem~\ref{thm:polyak}, we have
\begin{equation}\label{eq:bound2}
    H_{\xi\sqrt{d}/\mu}(\|\widehat{\bx}^{(t)} - \bx^\star\|) \leq \frac{\widetilde{R}(a,\beta,\delta)}{(t+2)^{1-\delta} - 2^{1-\delta}},
\end{equation} By the definition of Huber loss, it then follows by~\eqref{eq:bound2} that, if $\|\widehat{\bx}^{(t)} - \bx^\star\| \leq \xi\sqrt{d}/\mu$, we have
\begin{equation}\label{eq:8}
    \|\widehat{\bx}^{(t)} - \bx^\star\|^2 \leq \frac{2\widetilde{R}(a,\beta,\delta)}{(t+2)^{1-\delta} - 2^{1-\delta}}.
\end{equation} Otherwise, if $\|\widehat{\bx}^{(t)} - \bx^\star\| > \xi\sqrt{d}/\mu$, by~\eqref{eq:bound2}, we have
\begin{equation*}
    \frac{\xi\sqrt{d}\|\widehat{\bx}^{(t)} - \bx^\star\|}{2\mu} < \xi\sqrt{d}\|\widehat{\bx}^{(t)} - \bx^\star\|/\mu - \frac{\xi^2d}{2\mu^2} \leq \frac{\widetilde{R}(a,\beta,\delta)}{(t+2)^{1-\delta} - 2^{1-\delta}},
\end{equation*} implying that 
\begin{equation}\label{eq:9}
    \|\widehat{\bx}^{(t)} - \bx^\star\|^2 \leq \frac{4\mu^2\widetilde{R}(a,\beta,\delta)^2/(\xi^2 d)}{\left((t+2)^{1-\delta} - 2^{1-\delta}\right)^2}.
\end{equation} Combining~\eqref{eq:8} and~\eqref{eq:9}, it then follows that
\begin{equation*}
    \|\widehat{\bx}^{(t)} - \bx^\star\|^2 \leq \min\left\{\frac{2\widetilde{R}(a,\beta,\delta)}{(t+2)^{1-\delta} - 2^{1-\delta}}, \frac{4\mu^2\widetilde{R}(a,\beta,\delta)^2/(\xi^2 d)}{\left((t+2)^{1-\delta} - 2^{1-\delta}\right)^2} \right\},
\end{equation*} which, for $t$ sufficiently large, implies $\|\widehat{\bx}^{(t)} - \bx^\star\|^2 = \mathcal{O}\left(t^{2(\delta - 1)} \right)$.

\end{document}